\documentclass[11pt]{article}
\usepackage{amsmath,amssymb,amsthm,mathtools,bm,mathrsfs}

\makeatletter
\let\amsthm@newtheorem\newtheorem
\let\amsthm@proof\proof
\let\amsthm@endproof\endproof
\def\newtheorem#1{\@ifnextchar[\jmlr@skipthm\@gobble}
\def\jmlr@skipthm[#1]#2{}
\makeatother
\usepackage[preprint,nohyperref]{jmlr2e}
\makeatletter
\let\newtheorem\amsthm@newtheorem
\let\proof\amsthm@proof
\let\endproof\amsthm@endproof
\makeatother

\usepackage{booktabs,multirow,array,tabularx}
\usepackage{xcolor}
\usepackage[ruled,vlined,linesnumbered]{algorithm2e}
\SetAlgoNlRelativeSize{-1}
\usepackage[expansion=false]{microtype}
\usepackage{enumitem}
\usepackage{caption}
\usepackage[colorlinks=false,allbordercolors={1 1 1}]{hyperref}

\newtheorem{theorem}{Theorem}[section]
\newtheorem{lemma}[theorem]{Lemma}
\newtheorem{proposition}[theorem]{Proposition}
\newtheorem{corollary}[theorem]{Corollary}
\theoremstyle{definition}

\newtheorem{remark}[theorem]{Remark}

\allowdisplaybreaks
\newcommand{\R}{\mathbb{R}}
\newcommand{\E}{\mathbb{E}}

\newcommand{\Kset}{\mathcal{K}}
\newcommand{\Fs}{\mathcal{F}}
\newcommand{\Ucal}{\mathcal{U}}
\newcommand{\ip}[2]{\langle #1, #2\rangle}
\newcommand{\pos}[1]{\left[#1\right]_{+}}

\newcommand{\Reg}{\mathrm{Reg}}
\newcommand{\CCV}{\mathrm{CCV}}
\newcommand{\ICV}{\mathrm{ICV}}
\newcommand{\WCV}{\mathrm{WCV}}
\newcommand{\Gh}{\widehat{G}}

\newcommand{\KL}{\mathrm{KL}}
\newcommand{\TV}{\mathrm{TV}}
\newcommand{\Ind}{\mathbf{1}}

\DeclareMathOperator{\clip}{clip}
\DeclareMathOperator*{\argmin}{arg\,min}

\newcommand{\LEDGER}{\textnormal{\textsc{Ledger}}}

\title{Constrained Online Learning with Noisy Constraint Values}
\author{\name Vaneet Aggarwal \email vaneet@purdue.edu\\
\addr Purdue University\\
West Lafayette, IN 47907, USA}

\hypersetup{pdftitle={Constrained Online Learning with Noisy Constraint Values},
 pdfauthor={Vaneet Aggarwal},pdfsubject={Finite-variance constrained online learning},
 pdfkeywords={online convex optimization, noisy constraints, regret, budget violation}}

\begin{document}
\maketitle

\begin{abstract}
We study online convex optimization with time-varying constraints when the constraint
values are observed only through noise. In each round the learner picks an action, pays
a convex loss, and incurs a convex constraint value that should be nonpositive; it then
receives unbiased estimates of the loss and constraint gradients with bounded second
moments, and an unbiased observation of the constraint value whose error has bounded
variance. The errors may be correlated across the three observations, and the
constraint values need not be bounded below.
Performance is measured against the best fixed action in hindsight among those that
satisfy every constraint, and violation is measured by the largest total overspend over
any interval of consecutive rounds, which we call budget violation. Our algorithm,
\LEDGER, has $O(\sqrt T)$ expected regret and $O(\sqrt{T\log(eT)})$ expected budget
violation over $T$ rounds, uses one projection per round, and needs no strictly
feasible action. The logarithm cannot be
removed: with Gaussian noise of fixed positive variance, every algorithm with $O(\sqrt T)$
worst-case regret has worst-case budget violation at least of order $\sqrt{T\log T}$,
even if only the final total overspend is counted, and budget violation $O(\sqrt T)$
forces linear regret. When violation is instead the sum of the positive parts of the
per-round constraint values (hard violation), no algorithm keeps both expected regret and
expected hard violation sublinear: if the noise distribution is not fixed, a randomly
shifted constraint can produce the same observations as a fixed one observed with noise,
even for affine constraints. With a fixed and known Gaussian noise law
of standard deviation $\sigma$, a construction with two constraint distributions that
share many moments still gives a worst-case lower bound of order
$\min\{\sigma,1\}T(\log\log T/\log T)^2$, so regret and hard violation cannot both be
made polynomially sublinear. Noisy constraint values therefore still allow a signed
budget to be controlled, at the cost of a logarithmic factor, whereas the hard violation,
which charges every round separately, cannot be made polynomially sublinear together with
the regret.
\end{abstract}


\section{Introduction}
\label{sec:intro}

Constrained online convex optimization (COCO) models sequential resource allocation when
both costs and constraints change over time. On round $t$ the learner chooses $x_t$ from
a known convex set, incurs loss $\ell_t(x_t)$, and has net consumption $\gamma_t(x_t)=c_t(x_t)-b_t$, its resource use $c_t(x_t)$ minus an
allowance $b_t$. The action $x$ satisfies the round's constraint when $\gamma_t(x)\le0$.
Regret is measured against a fixed action, which may be chosen in hindsight, that
satisfies the constraint on every round; the assumption that such an action exists is
called \emph{common feasibility}. When the constraint functions are revealed after each
round, exponential-penalty methods keep the regret at $O(\sqrt T)$ and the sum of positive
constraint values at $O(\sqrt T\log T)$ \citep{sinha2024optimal,yu2026withoutslater}, and
both can be kept at $O(\sqrt T)$ in fixed dimension \citep{sarkar2026improved}. When constraint
values are observed with noise, an unbiased estimate of $\gamma_t(x_t)$ need not reveal
whether $x_t$ was feasible. Earlier work either observes constraint values exactly,
sometimes only through bandit feedback \citep{cao2019online,sarkar2026projectionfree}, or
treats the observations as samples of constraints whose distribution is fixed
\citep{yu2017online,wei2020online,yu2026withoutslater} or whose mean stays within a
bounded total distance of a fixed reference \citep{stradi2026truly}, with a comparator
that only needs to satisfy the mean constraint.

We study adversarial convex losses and constraints that the learner observes only at the
played action, through three quantities: the gradient of the loss, the gradient of the constraint, and the constraint
value. The gradients guide the update, but only the value shows how far the constraint is from
being satisfied: adding a constant to $\gamma_t$ leaves its gradient unchanged but changes
which actions are feasible and how much is overspent, whereas adding a constant to a loss
leaves regret unchanged. In each round the learner
receives one noisy estimate of each quantity; given the past, the current action, and the
current functions, the errors have mean zero
and bounded second moments, and they need not be independent, symmetric, or bounded.

The \emph{hard violation} $\ICV_T=\sum_t[\gamma_t(x_t)]_+$ charges every infeasible round
separately. When the noise law is not fixed, no algorithm keeps both regret and hard
violation sublinear, even for one-dimensional affine constraints with known gradients, and
with a fixed Gaussian noise law of positive variance they still cannot both be polynomially sublinear
(contribution~3 below). Many resources, however, can be balanced across rounds, and for them
the relevant measure is the \emph{budget violation}
\[
\WCV_T=\max\left\{0,\max_{1\le t_1\le t_2\le T}\sum_{t=t_1}^{t_2}\gamma_t(x_t)\right\},
\]
the largest total overspend over an interval of consecutive rounds. Underspending can
offset overspending inside an interval, but an early surplus cannot hide a later
overspend, because the interval may start after the surplus. The \emph{terminal violation} $\CCV_T=[\sum_t\gamma_t(x_t)]_+$ counts only the
final total and satisfies $\CCV_T\le\WCV_T\le\ICV_T$. For example, the alternating values
$-1,+1,-1,+1,\ldots$ give, for even $T$, terminal violation $0$, budget violation $1$,
and hard violation $T/2$. Budget violation suits resources that can be balanced across
rounds; it does not suit settings in which every infeasible action is unacceptable.
\citet[Sec.~3]{sinha2024optimal} use the same maximum over intervals for constraint satisfaction with exact feedback and
no loss.

\paragraph{When budget violation is the right measure.}
Budget violation is the size of the reserve that overspending requires. Suppose that
overspending is paid from a reserve of capacity $R$ that starts full and is refilled by
underspending, but never above $R$, as with a battery that buffers an energy budget or a
credit line that is repaid from later savings. The shortfall of the reserve then follows
the recursion $\Gamma_t=[\Gamma_{t-1}+\gamma_t(x_t)]_+$ of Section~\ref{sec:setting}, so the
reserve meets every withdrawal exactly when $R\ge\WCV_T$. This identity holds on each run,
so a bound on the expected budget violation bounds the expected reserve capacity that is
required. In such systems consumption is often
known only through noisy meter readings or sampled estimates, which is the value
feedback studied here. Among the three measures, budget violation is the strongest one
that can be controlled together with square-root regret: hard violation cannot
(contribution~3 below), and budget violation is at least the terminal violation yet has
the same minimax order (Section~\ref{sec:metric-consequences}).\par

\paragraph{Contributions.}
\begin{enumerate}[leftmargin=1.6em,itemsep=2pt,topsep=2pt]
\item \textbf{Square-root regret with optimal-order budget violation.}
Under common feasibility, \LEDGER\ (Algorithm~\ref{alg:ledger-exp}) has $O(\sqrt T)$
expected regret against the best feasible fixed action in hindsight and
$O(\sqrt{T\log(eT)})$ expected budget violation (Theorem~\ref{thm:exp});
Table~\ref{tab:results} compares these rates with those for exact feedback. Each round
uses one noisy observation of each of the three quantities and one projection onto the
decision set. The violation bound concerns the true constraint values, with the maximum over intervals inside the
expectation. Restarting the algorithm on intervals of doubling length removes the need to
know $T$ and keeps both rates (Corollary~\ref{cor:anytime}).
\item \textbf{A matching lower bound for budget violation.}
Two instances whose constraints differ only by a constant, observed through Gaussian
noise, give a tradeoff between regret and budget
violation (Proposition~\ref{prop:lower-wcv-log}). For a fixed noise standard deviation
$\sigma>0$, every algorithm
with $O(\sqrt T)$ worst-case expected regret has worst-case expected budget violation
$\Omega(\sigma\sqrt{T\log T})$, and budget violation $O(\sqrt T)$ forces linear regret.
In the worst case, the larger of the two is $\Omega(\min\{\sigma,1\}\sqrt{T\log T})$ for
every algorithm
(Corollary~\ref{cor:budget-joint}). These bounds hold for terminal violation as well,
with exact gradients, so the logarithm comes from the unknown constant offset of the constraint and
not from the maximum over intervals or from estimating gradients.
\item \textbf{Lower bounds for hard violation.}
We give two bounds, which come from different obstructions. If the noise law may
differ between instances, a constraint shifted by a random sign in each round produces
the same observations as a fixed constraint observed with noise, so the larger of expected
regret and expected hard violation is at least $\min\{\sigma,1\}T/6$, even with exact
gradients (Proposition~\ref{prop:lower-linear}). If the noise is Gaussian with a fixed and
known positive variance, the instances can be told apart, but only slowly: two shift
distributions that share many moments give a lower bound of order
$\min\{\sigma,1\}T(\log\log T/\log T)^2$, so regret and hard violation cannot both be
polynomially sublinear (Theorem~\ref{thm:lower-n}).
\end{enumerate}

\begin{table}[t]
\centering\small
\renewcommand{\arraystretch}{1.18}
\begin{tabularx}{\linewidth}{@{}>{\raggedright\arraybackslash}X
>{\raggedright\arraybackslash}p{2.35cm}
>{\raggedright\arraybackslash}p{3.25cm}@{}}
\toprule
\textbf{Feedback / guarantee} & \textbf{Regret} & \textbf{Violation}\\
\midrule
Exact; exponential-penalty methods\newline
\citep{sinha2024optimal,yu2026withoutslater}
& $O(\sqrt T)$ & hard: $O(\sqrt T\log T)$\\
Exact; projection onto the intersection of revealed constraint sets
\citep{sarkar2026improved}
& $O(\sqrt T)$ & hard: $O(\sqrt T)$ for fixed dimension\\
Noisy values and gradients;\newline
\LEDGER\ (Theorem~\ref{thm:exp})
& $O(\sqrt T)$ & budget: $O(\sqrt{T\log(eT)})$\\
Noisy values; necessary budget\newline
(Proposition~\ref{prop:lower-wcv-log})
& $O(\sqrt T)$ required & budget: $\Omega(\sigma\sqrt{T\log T})$\\
Noisy values; noise law not fixed\newline
(Proposition~\ref{prop:lower-linear})
& \multicolumn{2}{l}{larger of regret and hard: $\Omega(\min\{\sigma,1\}T)$}\\
Noisy values; Gaussian noise\newline
(Theorem~\ref{thm:lower-n})
& \multicolumn{2}{l}{larger of regret and hard:
$\Omega\big(\min\{\sigma,1\}T\big(\frac{\log\log T}{\log T}\big)^2\big)$}\\
\bottomrule
\end{tabularx}
\caption{\textbf{Joint guarantees under common feasibility.}
Rates are for fixed problem bounds, and the noisy-feedback bounds are in expectation.
Row 4 is a lower bound for
algorithms with $O(\sqrt T)$ worst-case regret, and rows 5 and 6 are worst-case lower
bounds for every algorithm. For fixed $\sigma>0$, rows 3 and 4 match in their dependence
on $T$.}
\label{tab:results}
\end{table}

\paragraph{Technical approach.}
\LEDGER\ keeps a running balance of observed overspend, resets it to zero whenever it
becomes negative, and weights the constraint gradient by a factor that grows
exponentially in the balance. With exact feedback, \citet{sinha2024optimal} use
exponential weights on accumulated positive violation, and, for stochastic constraints
observed exactly, \citet{yu2026withoutslater} use a signed balance without reset. With
noisy values two further changes are needed. First, the factor is the derivative of
$e^{\alpha y}-\alpha y-1$ at the balance $y$, for a parameter $\alpha>0$ of order $\sqrt{\log T/T}$; this function has
zero slope at $y=0$, so resetting the balance changes it only at second order. Second,
each observed value is truncated to $[-1/\alpha,1/\alpha]$ before it enters the balance,
so no exponential moments of the noise are needed. As in \citet{yu2026withoutslater}, a
small amount is also subtracted from the balance each round; multiplied by the factor, it
cancels the regret cost of step sizes that shrink as the factors grow, and its total adds
only $O(\sqrt{T\log(eT)})$ to the budget violation (Section~\ref{sec:proof-design}).

\paragraph{Related work.}
Online convex optimization with long-term constraints began with
\citet{mahdavi2012trading}. Adversarial time-varying constraints without a Slater
condition (a fixed action that satisfies every constraint with a positive margin) were
studied by \citet{yi2022regret,guo2022online}, with $O(T^{3/4})$ hard violation,
before the exact-feedback rates of Table~\ref{tab:results}. \citet{liakopoulos2019cautious} compare against actions feasible over
windows of fixed length; we instead bound the learner's own largest overspend over all
intervals. Safe online learning keeps every iterate feasible.
\citet{hamoud2025safety} obtain zero violation for slowly changing constraints that are
revealed exactly, under a strict-feasibility margin, and related methods keep a fixed
unknown constraint satisfied from exact or sub-Gaussian observations
\citep{usmanova2019safe,chaudhary2022safe,hutchinson2024firstorder,hutchinson2025optimistic}.
These methods use exact constraint observations or a known strictly feasible action.
Without either, a uniform $O(\sqrt T)$ regret guarantee requires worst-case budget
violation of order $\sigma\sqrt{T\log T}$, even for a fixed affine constraint and a known
feasible boundary action (Proposition~\ref{prop:lower-wcv-log}). Gaussian noise forces logarithmic rates when a density is recovered from
noisy samples \citep{fan1991optimal}, pairs of laws that share many moments are a standard
source of lower bounds in estimation \citep{wu2019chebyshev}, and estimating the largest
value of a distribution observed through additive noise
\citep{goldenshluger2004estimating,brunel2021estimation} is the statistical counterpart
of our hard-violation bound; we reduce online learning to such a test and construct the
distributions explicitly.

\section{Setting}
\label{sec:setting}

Write $[a]_+=\max\{a,0\}$; logarithms are natural. Fix a compact convex set
$\Kset\subseteq\R^d$ of Euclidean diameter at most $D>0$ and a horizon $T\ge1$.
On round $t$ the learner plays $x_t\in\Kset$ and incurs the loss $\ell_t(x_t)$ and the
constraint value $v_t=\gamma_t(x_t)$, where $\ell_t$ and $\gamma_t$ are convex and
$G$-Lipschitz on an open set containing $\Kset$. Any per-round allowance is absorbed into
$\gamma_t$, and $x$ is feasible on round $t$ when $\gamma_t(x)\le0$. The adversary may
choose $\ell_t$ and $\gamma_t$ based on the past and on $x_t$, but not on the current
feedback noise. The learner never sees the functions; it can compute the Euclidean
projection $\Pi_\Kset$ onto $\Kset$ and receives only the feedback defined below. We
assume that $(\omega,x)\mapsto\ell_t(x)$ and $(\omega,x)\mapsto\gamma_t(x)$ are jointly
measurable, where $\omega$ is the outcome of all randomness, so that $v_t$ is a random
variable; the subgradients below are taken to be measurable selections, which exist by
the Kuratowski--Ryll-Nardzewski theorem \citep[Thm.~18.13]{aliprantis2006infinite}.

Let $\Fs_t$ be the $\sigma$-field generated by everything before the round-$t$ feedback:
the history, $x_t$, the adversary's functions, and the subgradients
\begin{equation}
\label{eq:vt}
p_t\in\partial\ell_t(x_t),\qquad q_t\in\partial\gamma_t(x_t),\qquad v_t=\gamma_t(x_t).
\end{equation}
For known $\Gh\ge G>0$ and $\sigma\ge0$, the oracle returns
$(\tilde p_t,\tilde q_t,\tilde v_t)$ with
\begin{equation}
\label{eq:feedback}
\E[\tilde p_t\mid\Fs_t]=p_t,\quad
\E[\tilde q_t\mid\Fs_t]=q_t,\quad
\tilde v_t=v_t+\varepsilon_t,\quad \E[\varepsilon_t\mid\Fs_t]=0,
\end{equation}
\begin{equation}
\label{eq:finite-variance}
\E[\|\tilde p_t\|^2\mid\Fs_t],\ \E[\|\tilde q_t\|^2\mid\Fs_t]\le\Gh^2,
\qquad \E[\varepsilon_t^2\mid\Fs_t]\le\sigma^2.
\end{equation}
The three errors of a round may depend on each other, and their conditional
distributions may change with the history; they need not be symmetric, bounded, or
sub-Gaussian. Because the variance condition is imposed on the error $\varepsilon_t$
rather than on $\tilde v_t$, no bound on $|v_t|$ is needed. Setting $\sigma=0$ makes only
the value observation exact; exact feedback makes all three exact.

Set $\mathcal G_t=\sigma(\Fs_t,\tilde p_t,\tilde q_t,\tilde v_t)$ and
$\mathcal G_0=\Fs_1$, so that $\mathcal G_{t-1}\subseteq\Fs_t\subseteq\mathcal G_t$ and, by
the tower property, errors with conditional mean zero given $\Fs_t$ are martingale
differences for $(\mathcal G_t)$. A quantity computed from the learner's past is
$\mathcal G_{t-1}$-measurable, and we call it \emph{predictable}; all weights and step
sizes of our algorithm are predictable. The field $\Fs_t$ also contains the adversary's
current choices, which the learner does not observe, and the noise is centered
conditionally on them.

\paragraph{Comparators and metrics.}
Regret is measured on the realized sequence: a single action is evaluated on the functions
actually generated against the learner, and when the adversary adapts we do not compare
with the functions it would have produced had that action been played. Formally, a
\emph{comparator} is a random action $U\in\Kset$ that is feasible on every round,
\[
\gamma_t(U)\le0\quad\text{for all }t\le T,\quad\text{almost surely}.
\]
It is fixed in the sense that one action is compared with all rounds, but it may depend on
the realized functions and noise, so it may be chosen in hindsight. This is the standard
comparator of adversarial COCO \citep{sinha2024optimal}; the only difference under noisy
feedback is that it may also depend on the noise, which the analysis handles in
Section~\ref{sec:proof-design}. Let $\Ucal$ be the set
of comparators. We assume common feasibility, $\Ucal\ne\emptyset$; no margin of strict
feasibility is required. Every deterministic action that is almost surely feasible on
every round belongs to $\Ucal$. If the realized feasible set
$\Kset_T=\{x\in\Kset:\gamma_t(x)\le0\text{ for all }t\le T\}$ is nonempty almost surely,
then the measurable maximum theorem \citep[Thm.~18.19]{aliprantis2006infinite} gives a
measurable selection
\[
U_T^*\in\argmin_{u\in\Kset_T}\sum_{t=1}^T\ell_t(u):
\]
the functions are continuous in $x$ and jointly measurable, so, after modification on a
null set, $\omega\mapsto\Kset_T$ is a measurable correspondence with nonempty compact
values. This best feasible action in hindsight belongs to $\Ucal$, and
$\E\Reg_T(U_T^*)=\E\max_{u\in\Kset_T}\Reg_T(u)$ for the regret defined below, which is at
least the expected regret against any deterministic comparator. Define
\begin{align}
\Reg_n(U)&=\sum_{t=1}^n\big(\ell_t(x_t)-\ell_t(U)\big),
&\WCV_n&=\max\left\{0,\max_{1\le t_1\le t_2\le n}\sum_{t=t_1}^{t_2} v_t\right\},
\label{eq:metrics}\\
\ICV_n&=\sum_{t=1}^n[v_t]_+,
&\CCV_n&=\left[\sum_{t=1}^n v_t\right]_+.
\nonumber
\end{align}
These are the regret, budget violation, hard violation, and terminal violation; the
abbreviations stand for window, instantaneous, and cumulative constraint violation. Our
regret bounds concern expected regret, which can be negative. Several
COCO papers write CCV for the sum of positive parts
\citep{sinha2024optimal,sarkar2026improved}, which is our hard violation $\ICV_n$; our
$\CCV_n$ is the positive part of the total. Always
\begin{equation}
\label{eq:metric-order}
\CCV_n\le\WCV_n\le\ICV_n.
\end{equation}
The three violation metrics coincide when every $v_t\ge0$. Our budget bound controls
the expectation of the maximum over windows (intervals $[t_1,t_2]$ of consecutive rounds),
which is at least the maximum over
windows of the expected overspend. The Lindley recursion~\citep{lindley1952theory}
$\Gamma_0=0$, $\Gamma_t=[\Gamma_{t-1}+v_t]_+$ satisfies $\WCV_n=\max_{t\le n}\Gamma_t$,
because $\Gamma_t$ is the larger of $0$ and the largest sum over windows ending at $t$; resetting
a negative balance to zero starts a new window. Our algorithm maintains a version of this
balance built from noisy observations (Section~\ref{sec:ledger}).

\paragraph{A bounded value for the analysis.}
Constraint values may be arbitrarily negative, and the analysis replaces them by values
bounded below that still compare correctly with every comparator.

\begin{lemma}[A bounded value for analysis only]
\label{lem:anchor}
Set $L=D\Gh$ and $v_t^\circ=\max\{v_t,-L\}$. For every comparator $U\in\Ucal$, almost
surely,
\begin{equation}
\label{eq:effective-level}
-L\le v_t^\circ\le L,\qquad v_t\le v_t^\circ\le\langle q_t,x_t-U\rangle.
\end{equation}
\end{lemma}
\begin{proof}
Since $q_t\in\partial\gamma_t(x_t)$, convexity gives
$\gamma_t(y)\ge v_t+\langle q_t,y-x_t\rangle$ for every $y\in\Kset$; taking $y=U$ and
using $\gamma_t(U)\le0$ almost surely gives $v_t\le\langle q_t,x_t-U\rangle$ almost surely. By \eqref{eq:feedback}, conditional
Jensen, and \eqref{eq:finite-variance},
$\|q_t\|=\|\E[\tilde q_t\mid\Fs_t]\|\le(\E[\|\tilde q_t\|^2\mid\Fs_t])^{1/2}\le\Gh$.
With $\|x_t-U\|\le D$, Cauchy--Schwarz gives $|\langle q_t,x_t-U\rangle|\le L$. So both
$v_t$ and $-L$ are at most $\langle q_t,x_t-U\rangle\le L$, hence so is their maximum
$v_t^\circ=\max\{v_t,-L\}$, which also satisfies $v_t^\circ\ge v_t$ and
$v_t^\circ\ge-L$.
\end{proof}

In particular $v_t\le L$, while $v_t$ may be arbitrarily negative. The bounded value
$v_t^\circ$ appears only in the analysis of Section~\ref{sec:ledger}: replacing $v_t$ by
the larger $v_t^\circ$ preserves the comparison with $U$ and removes the need for a
lower bound on the constraint values.

\section{Statistical limits with noisy values}
\label{sec:lower}

All lower bounds use the decision set $\Kset=[0,1]$, the loss $\ell_t(x)=-x$,
and exact gradients, and the adversary fixes all functions in advance. In the
hard-violation constructions the constraints are random but drawn before play, and
expectations include this randomness. The bounds are stated for two classes of instances.
For $\sigma>0$, $\mathfrak C_\sigma$ is the class of instances on $\Kset=[0,1]$ that satisfy
Section~\ref{sec:setting} with $D=G=\Gh=1$, $|\gamma_t(x)|\le1$, and value-noise variance
at most $\sigma^2$; the law of the noise may differ between its instances. The subclass
$\mathfrak C_\sigma^{\mathrm G}\subseteq\mathfrak C_\sigma$ fixes that law: it contains the
instances with $\ell_t(x)=-x$, exact gradients, and value noise drawn independently from
$N(0,\sigma^2)$, so the learner knows the noise law.

\subsection{Hard violation}

The model of Section~\ref{sec:setting} bounds the conditional variance of the noise but
does not fix its law. A constraint that is shifted by a random amount in each round can
then produce the same observations as a fixed constraint observed with noise.

\begin{proposition}[Linear lower bound when the noise law is not fixed]
\label{prop:lower-linear}
Fix $\sigma>0$ and $T\ge1$, and let $\theta=\min\{\sigma,\frac12\}$. Let $R_t,S_t$ be
independent uniform random signs. In instance $A$, $\gamma_t(x)=x-\frac12$ and
$\varepsilon_t=\theta R_t$; in instance $B$, $\gamma_t(x)=x-\frac12+\theta S_t$, with $S_{1:T}$ drawn
before play, and $\varepsilon_t=0$. Both have $\ell_t(x)=-x$ and exact gradients, both
lie in $\mathfrak C_\sigma$ and satisfy common feasibility (the action $0$ is feasible on
every round in both, and the action $\frac12$ in $A$), and every learner satisfies
\begin{equation}
\label{eq:lower-linear}
\max\{\E_A\Reg_T(\tfrac12),\E_B\ICV_T\}\ge\frac{\theta T}3\ge\frac{\min\{\sigma,1\}T}6.
\end{equation}
\end{proposition}
The proof is in Appendix~\ref{app:lower-linear}.

In both instances $\tilde v_t=x_t-\frac12+\theta\cdot(\text{a fresh uniform sign})$ and the
gradients are the same, so the learner's actions have the same law in $A$ and $B$. In $A$
regret is lost by staying below $\frac12$, and in $B$ every action above $\frac12-\theta$
violates the constraint with probability at least $\frac12$; the proof combines the two
with the identity
$-y+[y-\theta]_++[y+\theta]_+=\max\{\theta,|y|\}$. The construction needs the noise law to
differ between the instances, since the randomness is noise in $A$ and part of the
constraint in $B$, so it does not apply when the noise law is fixed. For
$\gamma_t(x)=x-\frac12+\zeta_t$ with random $\zeta_t$, the largest action that is
always feasible is $\frac12$ minus the largest possible value of $\zeta_t$, so it depends
on the top of the support of $\zeta_t$ and not on its mean. Gaussian noise makes two
laws with different largest values but many equal moments nearly indistinguishable.

\begin{theorem}[Hard-violation lower bound with Gaussian noise]
\label{thm:lower-n}
For every $\sigma>0$, integer $\nu\ge1$, and horizon $T\ge1$, set
\begin{equation}
\label{eq:lower-scale}
s=\min\left\{\frac12,\frac\sigma2,
\sigma\left(\frac{(2\nu)!}{12T}\right)^{1/(4\nu)}\right\}.
\end{equation}
There are explicit symmetric laws $P_\nu,P_\nu'$ for $\zeta_t$, supported in $[-s,s]$,
defining two instances in $\mathfrak C_\sigma^{\mathrm G}$, and an action $u_\nu$ feasible
on every round under $P_\nu$, such that every learner satisfies
\begin{equation}
\label{eq:lower-n-finite}
\max\{\E_{P_\nu}\Reg_T(u_\nu),\E_{P_\nu}\ICV_T,\E_{P_\nu'}\ICV_T\}
\ge\frac{sT}{224\nu^2}.
\end{equation}
For fixed $\nu,\sigma$ and all sufficiently large $T$, this is at least
$\sigma T^{1-1/(4\nu)}/(448\nu^2)$.
For $T\ge3$, choosing the smallest $\nu$ such that $(2\nu)!16^\nu\ge12T$ gives
\begin{equation}
\label{eq:lower-growing}
\Omega\!\left(\min\{\sigma,1\}T
\left(\frac{\log\log T}{\log T}\right)^2\right),
\end{equation}
with absolute implied constants.
\end{theorem}
The proof is in Appendix~\ref{app:lower-n-proof}.

The two laws come from Chebyshev interpolation: their first $2\nu-1$ moments agree,
their largest values differ by at least $s/(16\nu^2)$, and each puts probability at least
$1/3$ on its largest value. Regret can be negative in this instance; the weighted inequality of
Lemma~\ref{lem:payoff} still turns the gap between the largest values into the bound,
provided the laws of the $T$ noisy observations under $P_\nu$ and $P_\nu'$ are within total
variation $1/2$. Even with a fixed and known Gaussian noise law, the theorem rules out
$O(T^{1-\delta})$ bounds on both regret and hard violation for each fixed $\delta>0$,
though not every sublinear rate.

The two hard-violation bounds come from different obstructions. In
Proposition~\ref{prop:lower-linear} the instances are not identifiable: two environments
with different feasible actions produce identical observations, because the noise law may
change between them. Theorem~\ref{thm:lower-n} fixes the Gaussian noise law, which removes
this freedom; the obstruction is then that the top of the support of $\zeta_t$, which
determines the feasible actions, is recovered only slowly from Gaussian-smoothed
observations. The second bound therefore holds even when the learner knows the noise law.

\subsection{Budget violation}

For budget violation we use the constant affine constraints
\begin{equation}
\label{eq:budget-pair}
\gamma_t^A(x)=\rho(x-1),\qquad\gamma_t^B(x)=\rho x,
\qquad 0<\rho\le1.
\end{equation}
The best feasible action is $1$ in $A$ and $0$ in $B$, and the two constraints have the
same gradient. Whatever the action, $\tilde v_t-\rho x_t$ is Gaussian with mean $-\rho$
in $A$ and $0$ in $B$, so the learner faces a test between two Gaussian means.

\begin{proposition}[Regret--budget tradeoff]
\label{prop:lower-wcv-log}
Fix $\sigma>0$ and a learner, and let $R_T$ and $\mathcal B_T$ be uniform bounds over
$\mathfrak C_\sigma^{\mathrm G}$: $\E_I\Reg_T(u)\le R_T$ and $\E_I\WCV_T\le\mathcal B_T$
for every instance $I\in\mathfrak C_\sigma^{\mathrm G}$ and every action $u$ feasible on
every round of $I$.
For $0<\mathcal B_T\le T/2$,
\begin{equation}
\label{eq:lower-wcv-log}
R_T\ge\frac T4\exp\!\left(-\frac{4\mathcal B_T^2}{\sigma^2T}\right).
\end{equation}
If, in addition, $0<R_T<T/4$, then
\begin{equation}
\label{eq:budget-inverse}
\mathcal B_T\ge\frac\sigma2\sqrt{T\log\!\left(\frac{T}{4R_T}\right)}.
\end{equation}
The same statements hold with terminal violation in place of budget violation.
\end{proposition}

\begin{proof}[Proof mechanism]
Write $\mathcal R_A=\E_A\Reg_T(1)$ and $\mathcal V_B=\E_B\WCV_T$.
Let $Y_t=\tilde v_t-\rho x_t$. The likelihood ratio $J=dP_B/dP_A$ of $(Y_{1:T},\varpi)$,
where $\varpi$ is the learner's internal randomness, satisfies
$\E_AJ^2=e^{T\rho^2/\sigma^2}$. Set $Z=T^{-1}\sum_t(1-x_t)\in[0,1]$. Then
$\E_A Z=\mathcal R_A/T$, whereas $\E_B Z=1-\mathcal V_B/(\rho T)$ because
$\WCV_T=\rho\sum_t x_t$ in $B$.
Changing measure gives $\E_B Z=\E_A[JZ]$; Cauchy--Schwarz and
$\E_A Z^2\le\E_A Z$ yield
\begin{equation}
\label{eq:budget-testing}
1-\frac{\mathcal V_B}{\rho T}=\E_A[JZ]
\le e^{T\rho^2/(2\sigma^2)}\sqrt{\mathcal R_A/T}.
\end{equation}
With $\rho=2\mathcal B_T/T$ the budget guarantee makes the left side at least $1/2$,
and squaring gives \eqref{eq:lower-wcv-log}. The full proof is in
Appendix~\ref{app:lower-budget-log}.
\end{proof}

\begin{corollary}[Matching joint order; square-root budget violation forces linear regret]
\label{cor:budget-joint}
For each $T\ge2$ and $\sigma>0$, a pair of instances of the form
\eqref{eq:budget-pair} in $\mathfrak C_\sigma^{\mathrm G}$ satisfies, for every learner,
\begin{equation}
\label{eq:joint-budget-lower}
\max\{\E_A\Reg_T(1),\E_B\WCV_T\}
\ge\frac{\min\{\sigma,1\}}4\sqrt{T\log T}.
\end{equation}
Under a uniform $R_T=O(\sqrt T)$ guarantee, Proposition~\ref{prop:lower-wcv-log} requires
$\mathcal B_T=\Omega(\sigma\sqrt{T\log T})$ for fixed $\sigma>0$.
Conversely, a uniform bound $\mathcal B_T\le C_B\sqrt T$, with fixed $C_B>0$,
forces $R_T\ge\frac14e^{-4C_B^2/\sigma^2}T$ for all sufficiently large $T$.
\end{corollary}

The proof is in Appendix~\ref{app:lower-budget-joint}. Both results are
worst-case bounds over $\mathfrak C_\sigma^{\mathrm G}$, in which $\rho$ may depend on $T$
while $D$, $G$, and $\Gh$ stay fixed, and Section~\ref{sec:metric-consequences} compares
them with Theorem~\ref{thm:exp}.

\section{Achievable guarantees with signed balances}
\label{sec:ledger}

\LEDGER\ (Algorithm~\ref{alg:ledger-exp}) sets every parameter from the horizon and the
known bounds $D$, $\Gh$, and $\sigma$.

\subsection{An exponentially weighted balance reset at zero}
\label{sec:exp-alg}

The positive part of an unbiased estimate is biased: if $v_t=0$ and the noise is
$N(0,\sigma^2)$, then $\E\pos{\tilde v_t}=\sigma/\sqrt{2\pi}$. A constraint weight computed
from the current noisy value can also be correlated with the current gradient error.
\LEDGER\ therefore adds up signed values and fixes all weights and step sizes before the
round's feedback arrives. Set
\begin{equation}
\label{eq:exp-bounds}
\begin{gathered}
L=D\Gh,\qquad B_v=(L^2+\sigma^2)^{1/2}\ge L,\qquad
\Phi_\alpha(y)=e^{\alpha y}-\alpha y-1\quad(y\ge0),\\
K_0=4+10L,\qquad\Lambda_T=\log((T+1)K_0T).
\end{gathered}
\end{equation}
Here $L$ bounds $|\ell_t(x)-\ell_t(y)|$ for $x,y\in\Kset$.

\begin{algorithm}[H]
\small
\DontPrintSemicolon
\KwIn{Horizon $T$, decision set $\Kset$, and known bounds $D,\Gh,\sigma$.}
$L\leftarrow D\Gh$;
$B_v\leftarrow\sqrt{L^2+\sigma^2}$;
$\Lambda_T\leftarrow\log((T+1)(4+10L)T)$\;
$\alpha\leftarrow\min\big\{1/(128B_v),\,\sqrt{\Lambda_T/(17T)}/B_v\big\}$\;
$Q_0\leftarrow0$, $H_0\leftarrow\alpha^2$; choose $x_1\in\Kset$\;
\For{$t=1,\dots,T$}{
  $w_t\leftarrow\alpha(e^{\alpha Q_{t-1}}-1)$;
  $H_t\leftarrow H_{t-1}+w_t^2$\;
  $\eta_t\leftarrow D/(2\Gh\sqrt{t+H_t})$;
  $r_t\leftarrow16\alpha B_v^2+8Lw_t/(\sqrt{H_t}+\sqrt{H_{t-1}})$\;
  play $x_t$; observe $(\tilde p_t,\tilde q_t,\tilde v_t)$\;
  $\bar v_t\leftarrow\max\{-1/\alpha,\min\{\tilde v_t,1/\alpha\}\}$\;
  $x_{t+1}\leftarrow\Pi_\Kset\big(x_t-\eta_t(\tilde p_t+w_t\tilde q_t)\big)$\;
  $Q_t\leftarrow\pos{Q_{t-1}+\bar v_t-r_t}$\;
}
\caption{\LEDGER: predictably regularized signed balance}
\label{alg:ledger-exp}
\end{algorithm}

Only the value is truncated; the gradient estimates are used as observed. The prescribed
$\alpha$ satisfies $0<\alpha\le(128B_v)^{-1}$. The
balance $Q_t$ records regularized overspend, the weight $w_t=\Phi_\alpha'(Q_{t-1})$ turns
it into a penalty on the constraint, and $H_t$ accumulates the squared weights that set
the step size $\eta_t$; the \emph{regularizer} $r_t$ is subtracted from the balance.
Each round projects onto the fixed set $\Kset$, never onto a set defined by the
constraints. \citet{sinha2024optimal} used exponential potentials for COCO with exact
feedback, and the adaptive part of $r_t$ follows \citet[Alg.~1]{yu2026withoutslater}.

\paragraph{The two parts of $r_t$.}
The constant part $16\alpha B_v^2$ pays for the truncation bias and for the
second-order term, proportional to $w_t$, in the change of the potential. The cost of the
adaptive step sizes grows like $\sqrt{H_t}$ and is paid by the second part, which
vanishes when $w_t=0$ and is chosen so that $\frac12w_t$ times it equals
$4L(\sqrt{H_t}-\sqrt{H_{t-1}})$; $H_0=\alpha^2>0$ keeps its denominator positive. Because
the balance is lowered by $r_t$, the bound on true overspend adds the total of the $r_t$
back, so the proof of Section~\ref{sec:proof-design} bounds the $r_t$ twice, weighted by
$w_t$ for the regret and unweighted for the budget violation.

\subsection{Square-root regret and optimal-order budget control}
\label{sec:exp-result}

\begin{theorem}[\LEDGER: regret and budget guarantee]
\label{thm:exp}
Assume \eqref{eq:feedback}, \eqref{eq:finite-variance}, and $\Ucal\ne\emptyset$, and run
Algorithm~\ref{alg:ledger-exp} with $L,B_v,K_0,\Lambda_T$ as in \eqref{eq:exp-bounds} and
any fixed $0<\alpha\le(128B_v)^{-1}$. For every comparator $U\in\Ucal$ and every $n\le T$,
\begin{equation}
\label{eq:exp-master}
\E\Reg_n(U)+\E\Phi_\alpha(Q_n)+8\alpha B_v^2\E\sum_{t=1}^n w_t
\le4L\sqrt n+4L\alpha+202\alpha^2B_v^2n,
\end{equation}
and
\begin{equation}
\label{eq:exp-wcv}
\E\WCV_T\le\frac{\Lambda_T}{\alpha}+17\alpha B_v^2T
+8L\sqrt{2T\Lambda_T}+4\sigma\sqrt T.
\end{equation}
For the algorithm's prescribed $\alpha$,
\begin{equation}
\label{eq:exp-rates}
\begin{aligned}
\E\Reg_T(U)&\le4L\sqrt T+\tfrac1{32}+\tfrac{202}{17}\Lambda_T=O(\sqrt T),\\
\E\WCV_T&\le F_T+8L\sqrt{2T\Lambda_T}+4\sigma\sqrt T=O\!\left(\sqrt{T\log(eT)}\right),
\end{aligned}
\end{equation}
where $F_T=2B_v\sqrt{17T\Lambda_T}$ if $T\ge\frac{16384}{17}\Lambda_T$, and
$F_T=128B_v\Lambda_T+\frac{17}{128}B_vT$ otherwise.
The constants depend only on $(\Gh,D,\sigma)$. The bounds
allow an adaptive adversary and dependent errors in the three observations, without a
Slater condition.
\end{theorem}
Section~\ref{sec:proof-design} outlines the proof of Theorem~\ref{thm:exp}, and
Appendix~\ref{app:exp} gives it in full.

Taking $U=U_T^*$ (Section~\ref{sec:setting}) shows that
$\E\max_{u\in\Kset_T}\Reg_T(u)$, the expected regret against the best feasible action in
hindsight, is $O(\sqrt T)$. The regret bound has an additive
$O(\log(eT))$ term but no logarithmic factor on $\sqrt T$.

By Proposition~\ref{prop:lower-wcv-log}, the dependence of the budget bound on $T$, including
the logarithm, is the best possible for algorithms with $O(\sqrt T)$ worst-case regret when
the problem bounds and the Gaussian noise level $\sigma>0$ are fixed.

\begin{corollary}[Unknown horizon]
\label{cor:anytime}
Run Algorithm~\ref{alg:ledger-exp} on consecutive epochs of planned lengths
$1,2,4,\ldots$, restarting its state and using each epoch length in place of $T$. Under
the assumptions of Theorem~\ref{thm:exp}, for every deterministic horizon $T$ and every
comparator $U$ for the first $T$ rounds, $\E\Reg_T(U)=O(\sqrt T)$ and
$\E\WCV_T=O(\sqrt{T\log(eT)})$.
\end{corollary}
The proof is in Appendix~\ref{app:anytime}. Regret adds over epochs. For budget violation,
any window of rounds meets each epoch in a single sub-window, so $\WCV_T$ is at most the
sum of the per-epoch budget violations. Since epoch lengths double, both sums are
dominated by the last epoch, whose length is at most $T$.

\subsection{\texorpdfstring{Proof outline of Theorem~\ref{thm:exp}}{Proof outline of Theorem 4.1}}
\label{sec:proof-design}

\paragraph{Resetting without a first-order cost.}
The potential $\Phi_\alpha$ has two properties that the proof uses: $\Phi_\alpha'(0)=0$,
which removes any first-order cost of resetting the balance, and
$\Phi_\alpha'(y)=\alpha(e^{\alpha y}-1)\ge0$, which is the predictable weight $w_t$ on the
constraint gradient. Extend the potential by $F(y)=\Phi_\alpha([y]_+)$. Because
$\Phi_\alpha'(0)=0$, $F$ is continuously differentiable, and Taylor's formula gives, for
$y\ge0$ and $\alpha h\le1$,
\begin{equation}
\label{eq:proof-taylor}
F(y+h)-F(y)\le\alpha(e^{\alpha y}-1)h+\tfrac32\alpha^2e^{\alpha y}h^2,
\end{equation}
including steps that cross zero, where the reset enters only through the quadratic term.
We apply \eqref{eq:proof-taylor} with the truncated observation, which ensures
$\alpha h\le1$, after raising the value to the bounded $v_t^\circ$ of
Lemma~\ref{lem:anchor}; since $F$ is nondecreasing, this can only increase the change.
Half of $w_tr_t$ absorbs the truncation bias and the second-order terms proportional to
$w_t$, and the conditional drift satisfies (Lemma~\ref{lem:exp-drift})
\begin{equation}
\label{eq:proof-drift}
\E[\Phi_\alpha(Q_t)-\Phi_\alpha(Q_{t-1})\mid\Fs_t]
\le w_tv_t^\circ-\tfrac12w_tr_t+202\alpha^2B_v^2.
\end{equation}
A potential with nonzero slope at $0$ would instead pay a first-order cost at every reset in this argument.

\paragraph{Comparison with a comparator chosen in hindsight.}
Online gradient descent with the predictable step sizes $\eta_t$ and the direction
$d_t=\tilde p_t+w_t\tilde q_t$ gives a pathwise bound on $\sum_{t\le n}\ip{d_t}{x_t-u}$ that
holds for every $u\in\Kset$, hence for $u=U$. Because $U$ may depend on the round-$t$
errors, $\E\ip{\xi_t}{x_t-U}$ with $\xi_t=d_t-\mu_t$ and $\mu_t=p_t+w_tq_t$ need not
vanish. Following \citet[proof of Lemma~3.1, Lemma~6.1]{nemirovski2009robust}, we run, for
the analysis only, the projected sequence $z_1=x_1$, $z_{t+1}=\Pi_\Kset(z_t+\eta_t\xi_t)$, and
split
\begin{equation}
\label{eq:proof-split}
\ip{\mu_t}{x_t-U}=\ip{d_t}{x_t-U}+\ip{\xi_t}{U-z_t}+\ip{\xi_t}{z_t-x_t}.
\end{equation}
Summed over $t\le n$, the first two terms are bounded pathwise by projection inequalities
that hold for all comparator points at once, and the last has mean zero because $z_t-x_t$
is $\mathcal G_{t-1}$-measurable.
With convexity and Lemma~\ref{lem:anchor}, this gives (Lemma~\ref{lem:exp-primal})
\begin{equation}
\label{eq:proof-primal}
\E\Reg_n(U)+\E\sum_{t\le n}w_tv_t^\circ\le4L\sqrt n+4L\E\sqrt{H_n},
\end{equation}
twice the bound for a comparator fixed in advance. By its definition, $r_t$ satisfies
$\frac12w_tr_t=8\alpha B_v^2w_t+4L(\sqrt{H_t}-\sqrt{H_{t-1}})$, so summing
\eqref{eq:proof-drift} over $t\le n$ and adding \eqref{eq:proof-primal} cancels
$4L\E\sqrt{H_n}$ and proves \eqref{eq:exp-master}.

\paragraph{From the balance to true windows.}
Since $\Reg_n(U)\ge-Ln$, \eqref{eq:exp-master} gives $\E e^{\alpha Q_n}\le K_0T$, and
Jensen's inequality gives $\E\max_{n\le T}Q_n\le\Lambda_T/\alpha$. The update gives
$\sum_{t=t_1}^{t_2}\bar v_t\le Q_{t_2}-Q_{t_1-1}+\sum_{t=t_1}^{t_2}r_t$ on every window,
so, pathwise,
\begin{equation}
\label{eq:proof-transfer}
\WCV_T\le\max_{n\le T}Q_n+\sum_{t=1}^T r_t+\sum_{t=1}^T[\tilde v_t-\bar v_t]_+
+2\max_{n\le T}\Big|\sum_{s=1}^n\varepsilon_s\Big|.
\end{equation}
The truncation residuals cost at most $\alpha B_v^2T$ in expectation, and Doob's
inequality bounds the noise term by $4\sigma\sqrt T$ without a union bound over windows.
The adaptive regularizers have total at most $8L\sqrt{T\log(H_T/\alpha^2)}$, and the same
estimate of $\E e^{\alpha Q_t}$ controls this logarithm, so $\E\sum_tr_t\le16\alpha B_v^2T+8L\sqrt{2T\Lambda_T}$
(Lemmas~\ref{lem:exp-max} and~\ref{lem:exp-reg-cost}); this proves \eqref{eq:exp-wcv}.
The maximum over $n$ contributes to the logarithm, but
Proposition~\ref{prop:lower-wcv-log} shows that a logarithm is needed even for terminal
violation.

\subsection{Consequences for the violation metric}
\label{sec:metric-consequences}

Fix $\sigma>0$ and the class $\mathfrak C_\sigma$ of Section~\ref{sec:lower}. For each
violation metric $V\in\{\CCV,\WCV,\ICV\}$, let $\mathfrak M_T(V)$ be the minimax value of
the larger of the expected regret and the expected violation,
\[
\mathfrak M_T(V)=\inf_{\mathcal A}\sup_{I\in\mathfrak C_\sigma}\sup_{U\in\Ucal_I(\mathcal A)}
\max\{\E_I\Reg_T^{\mathcal A}(U),\E_I V_T^{\mathcal A}\}.
\]
The infimum is over learners that receive the feedback of Section~\ref{sec:setting}, and
$\Ucal_I(\mathcal A)$ is the set of comparators of Section~\ref{sec:setting} for the
instance $I$ played against the learner $\mathcal A$. This set depends on the learner
because the adversary may adapt, and instances without a comparator do not enter. The
constants in $\Theta_\sigma$ below depend on the fixed $\sigma>0$. Then, for $T\ge2$,
\begin{equation}
\label{eq:metric-minimax}
\begin{gathered}
\mathfrak M_T(\CCV)=\Theta_\sigma(\sqrt{T\log T}),\qquad
\mathfrak M_T(\WCV)=\Theta_\sigma(\sqrt{T\log T}),\\
\mathfrak M_T(\ICV)=\Theta_\sigma(T).
\end{gathered}
\end{equation}
The upper bounds for $\CCV$ and $\WCV$ follow from $\CCV_T\le\WCV_T$ and
Theorem~\ref{thm:exp}, and their lower bounds from Corollary~\ref{cor:budget-joint},
because the two metrics coincide in its violating instance. \LEDGER\ therefore controls
the stronger of the two at no cost in order. For $\ICV$, both regret and hard violation
are at most $T$ in $\mathfrak C_\sigma$, and Proposition~\ref{prop:lower-linear} gives the
lower bound $\min\{\sigma,1\}T/6$, whose dependence on $\sigma$ may not be sharp.

\section{Conclusion}
\label{sec:conclusion}

With noisy constraint values, no algorithm keeps both regret and hard violation sublinear
when the noise law is not fixed, and, for any fixed $\delta>0$, none keeps both at $O(T^{1-\delta})$ even when the
noise is Gaussian with a known positive variance. In contrast, \LEDGER\ keeps the regret against
the best feasible action in hindsight at $O(\sqrt T)$ and the budget violation at
$O(\sqrt{T\log(eT)})$, a dependence on $T$ that is optimal for fixed $\sigma>0$ even for
terminal violation.

\bibliography{refs}

\appendix
\section{Related work in detail}
\label{app:related}

\paragraph{COCO with exact constraint observation.}
Online convex optimization with long-term constraints was introduced by
\citet{mahdavi2012trading}, who traded regret against the cost of projecting onto a fixed
constraint set and obtained $O(\sqrt T)$ regret with $O(T^{3/4})$ violation of the \emph{signed}
long-term constraint; refinements for fixed constraints appear in
\citet{jenatton2016adaptive,yuan2018online,yu2020low}. For time-varying constraints,
\citet{neely2017online} use Slater's condition. In contrast, the distributed
time-varying-constraint guarantees of \citet{yi2022regret} do not require Slater's condition:
their tradeoff gives $O(\sqrt T)$ static network regret and $O(T^{3/4})$ hard violation
summed over the agents when their tradeoff parameter equals $1/2$. In the centralized adversarial
setting, \citet{guo2022online} obtain $O(\sqrt T)$ regret and $O(T^{3/4})$ hard
violation under common feasibility, also without Slater's condition.
\citet{sinha2024optimal} gave the first near-minimax-optimal algorithm, $O(\sqrt T)$ regret and
$O(\sqrt T\log T)$ hard violation, by running a Lipschitz-adaptive online subroutine on surrogate
losses that combine the loss with the constraint through an exponential Lyapunov potential.
Dimension- and instance-dependent refinements, optimistic algorithms, and projection-free
interfaces have also been studied
\citep{sinha2025beyond,vaze2025sqrtT,balasundaram2026breaking,lekeufack2024optimistic,lu2026bagel}.
Using nested feasible-set projections and a self-contraction argument,
\citet{sarkar2026improved} obtain $O(\sqrt T)$ regret and hard violation for general convex
losses in fixed dimension, when the constraint functions are revealed.
These methods use exact values or information about the feasible set, which our model does
not provide. Bandit feedback has also been studied \citep{cao2019online,sarkar2026projectionfree},
but there the constraint values are observed exactly and only gradients must be estimated
from them.

\paragraph{Stochastic and i.i.d.\ constraints.}
\citet{mannor2009online} showed that sublinear regret and violation need not be jointly
attainable against adversarial constraints when the fixed comparator is required to be feasible
only in the long run. \citet{yu2017online} studied i.i.d.\ constraints observed after
the action and obtained $O(\sqrt T)$ regret and violation via a virtual-queue primal--dual method
under Slater's condition; \citet{wei2020online} weakened Slater's condition to a bounded-Lagrange-multiplier
assumption. More recently, \citet{yu2026withoutslater} obtain $O(\sqrt T)$ expected regret and an
$O(\sqrt T\log T)$ bound on the expected signed cumulative violation $\sum_t\gamma_t(x_t)$ for
stochastic constraints without Slater's condition, using adaptive dual
regularization. These methods observe a constraint drawn
from a fixed distribution exactly, whereas our constraints may change adversarially and are
seen only through noisy values and gradients at the played action.
\citet[Thm.~5]{yu2026withoutslater} also treat adversarial constraints, obtaining
$O(\sqrt T)$ regret and $O(\sqrt T\log T)$ hard violation with exact constraint feedback.
That extension, their Algorithm~2, adds positive parts of exact values to the balance
(and subtracts their regularizer), whereas their Algorithm~1 for stochastic constraints
uses a signed balance without reset. \LEDGER\ builds on the adaptive regularization of
Algorithm~1 and differs from it by resetting the balance at zero, using a potential with
zero slope at zero, truncating the observed values, and allowing noisy values and
gradients.
\citet{bernasconi2024beyond} handle stochastic and adversarial constraints in bandits with
optimistic estimates of the constraints, with a benchmark that satisfies the constraints in
expectation. \citet{stradi2026truly} study bandits with constraints whose mean vectors
$\bar g_t$ deviate from a fixed reference vector by a total $\ell_1$ distance $C$ over the
rounds and assume a strategy that
satisfies the mean constraints with a positive margin. With full constraint feedback they
obtain $\tilde O(\sqrt T+C)$ regret and positive violation; with bandit constraint
feedback the violation bound is the same and the regret becomes $\tilde O(\sqrt T+C\sqrt T)$.
Their positive violation applies the positive part to the mean constraint at the
learner's mixed strategy $\pi_t$, that is, $\sum_t[\bar g_t^\top\pi_t]_+$, and they note that the analogous
sum over realized constraints can grow linearly; their benchmark only has to satisfy the
mean constraints summed over rounds. Our comparator must satisfy every realized
constraint, and our hard violation takes positive parts of realized values;
Proposition~\ref{prop:lower-linear} and Theorem~\ref{thm:lower-n} show that with noisy values
regret and hard violation cannot both be sublinear when the noise law may vary, or both
polynomially sublinear with a fixed Gaussian noise law, even when the constraint distribution is fixed. A comparator
that only had to satisfy the mean constraints would enlarge the comparator set and make
regret harder, so the two settings are separated by where the positive part is applied.
\citet{liakopoulos2019cautious} require the benchmark to be feasible only on average
over windows of a fixed length, a condition on the comparator, whereas budget violation
is a property of the learner's own actions. \citet{sinha2024banditq} studies a fair
prediction problem with a guaranteed minimum rate of reward for each arm, under full
information and bandit feedback.

\paragraph{Budgeted and knapsack formulations.}
\citet{castiglioni2022unifying} give a black-box framework for long-term-constraint problems
covering stochastic and adversarial feedback, with a global comparator whose adversarial reward
fraction depends on a strict-feasibility parameter. \citet{stradi2025noregret} show that a
prescribed spending plan suffices when rewards and costs are drawn from distributions that
change over time and the plan prescribes expected use; our allowance is absorbed into each
round's adversarial constraint, with a comparator feasible on every round. In bandits with
knapsacks, the stochastic formulations compare against a linear-programming or convex
benchmark \citep{badanidiyuru2018bandits,agrawal2014bandits} and the adversarial
formulation uses a competitive ratio \citep{immorlica2022adversarial}; neither is directly
comparable to our additive regret against a comparator feasible on every round.

\paragraph{Noise, deconvolution and lower bounds.}
Our hard-violation lower bounds are Le Cam two-point arguments~\citep{tsybakov2009introduction} in which the two
worlds are matched to high order. Theorem~\ref{thm:lower-n} is related to nonparametric
deconvolution, where supersmooth (Gaussian) errors lead to logarithmic rather than
polynomial rates~\citep{fan1991optimal}, and even more closely to estimating the endpoint
of a distribution from observations with additive errors
\citep{goldenshluger2004estimating,brunel2021estimation}: in our construction the best
feasible comparator is determined by the top of the support of $\zeta_t$, which the learner
sees only through Gaussian noise. Chebyshev approximation and moment matching have also been used to construct statistically
indistinguishable laws in other estimation problems~\citep{wu2019chebyshev}.
We interpolate at the Chebyshev extrema and evaluate the interpolant at a point just
outside $[-1,1]$; the
interpolation coefficients give two explicit symmetric laws with matching moments whose
largest values are separated and carry constant probability.
A weighted payoff inequality turns this gap into a lower bound on regret and hard
violation. The interpolation identities are classical
\citep{trefethen2013approximation}; the resulting distributions and online reduction
are proved in Appendix~\ref{app:lower}.
For budget violation, Proposition~\ref{prop:lower-wcv-log} instead bounds the second
moment of a Gaussian likelihood ratio for constant affine constraints.

\paragraph{Safe learning.}
A related literature keeps a stochastic constraint satisfied at every round with high
probability, using confidence bounds and a pessimistic feasible region: see
\citet{usmanova2019safe} for safe convex optimization with an unknown constraint, and
\citet{chaudhary2022safe} and \citet{hutchinson2025optimistic} for online convex
optimization with unknown linear constraints observed through sub-Gaussian noise; the
latter obtain near-square-root regret with no violation with high probability. These
results use a fixed constraint and, crucially, a known action that is strictly feasible
with a positive margin (in \citet{hutchinson2025optimistic}, the origin with positive
constraint limits), from which a pessimistic estimate can be built and refined. \citet{hutchinson2024firstorder}
observe the value and gradient of a fixed, smooth, and strongly convex unknown constraint
exactly at the played actions and, starting from a known feasible action, obtain
$O(\sqrt T)$ regret with no violation; they also treat monotone time-varying constraint
sets, with dynamic regret.
\citet{hamoud2025safety} let the constraint change by a bounded amount between consecutive
rounds, observe it exactly after each round, and obtain zero violation with dynamic regret
bounds, assuming a strict-feasibility margin in every round and a known feasible starting
action. A fixed
constraint alone does not explain the difference, since our budget lower bound also uses
fixed affine constraints. In its world $B$, $\gamma_t(x)=\rho x$ on $[0,1]$, no action is
strictly feasible: the action $0$, which may be disclosed to the learner, lies on the
boundary, and Proposition~\ref{prop:lower-wcv-log} shows that, for every algorithm with
a uniform $O(\sqrt T)$ regret guarantee, worst-case budget violation of order
$\sigma\sqrt{T\log T}$ is then unavoidable. The adversarial case, in addition, has no fixed function to estimate.

\section{\texorpdfstring{Proofs of the results in Section~\ref{sec:lower}}{Proofs of the results in Section 3}}
\label{app:lower}

This appendix proves Proposition~\ref{prop:lower-linear}, Theorem~\ref{thm:lower-n},
Proposition~\ref{prop:lower-wcv-log}, and Corollary~\ref{cor:budget-joint}.

\subsection{\texorpdfstring{Proof of Proposition~\ref{prop:lower-linear}}{Proof of Proposition 3.1}}
\label{app:lower-linear}

\begin{proof}[Proof of Proposition~\ref{prop:lower-linear}]
We first check that both instances lie in $\mathfrak C_\sigma$. In both, $\Kset=[0,1]$,
so $D=1$; the functions $\ell_t(x)=-x$ and $\gamma_t$ have slopes $-1$ and $1$, so $G=1$;
and the oracle returns the exact gradients $\tilde p_t=-1$ and $\tilde q_t=1$, whose
squared norms equal $1=\Gh^2$. Since $0\le x\le1$ and $\theta\le\frac12$,
$|\gamma_t(x)|\le|x-\frac12|+\theta\le1$. In $A$, $R_t$ is independent of everything
before the round-$t$ feedback, so $\E[\varepsilon_t\mid\Fs_t]=\theta\E R_t=0$ and
$\E[\varepsilon_t^2\mid\Fs_t]=\theta^2\le\sigma^2$; in $B$, $\varepsilon_t=0$. In $A$ the
action $\frac12$ satisfies $\gamma_t(\frac12)=0$; in $B$ the action
$\frac12-\theta\in[0,\frac12]$ satisfies $\gamma_t(\frac12-\theta)=\theta(S_t-1)\le0$; and
$\gamma_t(0)\le-\frac12+\theta\le0$ in both.

Next we show that the actions have the same law in both instances. Let $\varpi$ be the
learner's internal randomness, independent of the signs. The gradients carry no
information, so $x_t=\pi_t(\varpi,\tilde v_1,\ldots,\tilde v_{t-1})$ for fixed measurable
maps $\pi_t$. In $A$, $\tilde v_t=x_t-\frac12+\theta R_t$, and in $B$,
$\tilde v_t=\gamma_t(x_t)=x_t-\frac12+\theta S_t$. Define recursively, for signs
$e_1,\ldots,e_T$,
\begin{equation}
\label{eq:linear-recursion}
\chi_1=\pi_1(\varpi),\qquad
\chi_t=\pi_t\big(\varpi,\,\chi_1-\tfrac12+\theta e_1,\ldots,\chi_{t-1}-\tfrac12+\theta e_{t-1}\big).
\end{equation}
By induction on $t$, the actions are $x_t=\chi_t(\varpi,R_{1:t-1})$ in $A$ and
$x_t=\chi_t(\varpi,S_{1:t-1})$ in $B$. The pairs $(\varpi,R_{1:T})$ and $(\varpi,S_{1:T})$
have the same law, so $(x_1,\ldots,x_T)$ has the same law in $A$ and $B$.

It remains to compare the regret in $A$ with the hard violation in $B$. Write
$y_t=x_t-\frac12$. In $A$,
\begin{equation}
\label{eq:linear-regret}
\Reg_T(\tfrac12)=\sum_{t=1}^T\big(-x_t+\tfrac12\big)=-\sum_{t=1}^Ty_t.
\end{equation}
In $B$, $x_t$ is a function of $(\varpi,S_{1:t-1})$ and hence independent of $S_t$, so
conditioning on $x_t$ and averaging over $S_t=\pm1$ gives
\begin{equation}
\label{eq:linear-icv}
\E_B[\gamma_t(x_t)]_+=\E_B[y_t+\theta S_t]_+=\E_BV(y_t),\qquad
V(y)=\tfrac12\big([y-\theta]_++[y+\theta]_+\big).
\end{equation}
For every real $y$,
\begin{equation}
\label{eq:linear-identity}
-y+2V(y)=-y+[y-\theta]_++[y+\theta]_+=\max\{\theta,|y|\}\ge\theta.
\end{equation}
Indeed, the middle expression equals $-y$ for $y\le-\theta$, $\theta$ for
$-\theta\le y\le\theta$, and $y$ for $y\ge\theta$. Because $y_t$ has the same law in $A$
and $B$, $\E_Ay_t=\E_By_t$, and \eqref{eq:linear-regret}--\eqref{eq:linear-identity} give
\begin{equation}
\label{eq:linear-sum}
\E_A\Reg_T(\tfrac12)+2\E_B\ICV_T
=\sum_{t=1}^T\E_B\big[-y_t+2V(y_t)\big]\ge\theta T.
\end{equation}
The left side is at most $3\max\{\E_A\Reg_T(\frac12),\E_B\ICV_T\}$, which proves the first
inequality in \eqref{eq:lower-linear}. The second follows from
$\theta=\min\{\sigma,\frac12\}\ge\frac12\min\{\sigma,1\}$.

The constraints of $B$ are random, but they are drawn before play and $\frac12-\theta$ is
feasible for every sign sequence. Since $\E_B\ICV_T$ is the average over sign sequences of
the expected hard violation given the signs, some fixed sign sequence gives a
deterministic instance, chosen in advance, whose expected hard violation is at least
$\E_B\ICV_T$ (the sign sequence may depend on the learner); so \eqref{eq:lower-linear}
also holds with $B$ replaced by this deterministic instance.
\end{proof}

\subsection{\texorpdfstring{Auxiliary lemmas for the proof of Theorem~\ref{thm:lower-n}}{Auxiliary lemmas for the proof of Theorem 3.2}}
\label{app:lower-aux}

The proof of Theorem~\ref{thm:lower-n} combines Gaussian smoothing with a weighted payoff
inequality. The smoothing calculation is a standard moment-matching method
\citep{fan1991optimal,wu2019chebyshev}; the reduction to testing uses total variation
and Pinsker's inequality~\citep[Ch.~2]{tsybakov2009introduction}.
Throughout, $\phi_a$ is the density of $N(a,\sigma^2)$, $\phi=\phi_0$, and $P*\phi$
is the density of $\zeta+\varepsilon$ for independent $\zeta\sim P$ and
$\varepsilon\sim N(0,\sigma^2)$.

\begin{lemma}[Standard divergence facts]
\label{lem:div-facts}
For probability densities $p,q$ on a common space with $q>0$, write
\[
\KL(p\|q)=\int p\log\frac pq,\qquad
\chi^2(p\|q)=\int\frac{(p-q)^2}{q}=\int\frac{p^2}{q}-1,
\]
and $\TV(p,q)=\sup_{E}\big|\int_Ep-\int_Eq\big|=\frac12\int|p-q|$. Then:
\begin{enumerate}[label=(F\arabic*),leftmargin=2.4em,itemsep=2pt]
\item \label{fact:kl-chi}
$\KL(p\|q)\le\chi^2(p\|q)$.
\item \label{fact:kl-product}
(Tensorization) If $p=p_1\otimes\cdots\otimes p_T\otimes r$ and
$q=q_1\otimes\cdots\otimes q_T\otimes r$, then
$\KL(p\|q)=\sum_{t=1}^T\KL(p_t\|q_t)$.
\item \label{fact:pinsker}
(Pinsker) $\TV(p,q)\le\sqrt{\KL(p\|q)/2}$.
\item \label{fact:tv-test}
If a statistic $\Xi$ takes values in $[0,K]$, then
$|\E_p\Xi-\E_q\Xi|\le K\,\TV(p,q)$.
\end{enumerate}
\end{lemma}
\begin{proof}
These facts are classical~\citep[Ch.~2]{tsybakov2009introduction}; we include short
arguments for \ref{fact:kl-chi}, \ref{fact:kl-product}, and \ref{fact:tv-test}.
For \ref{fact:kl-chi}, because $\log$ is concave, Jensen's inequality under $p$
gives $\KL(p\|q)=\int p\log(p/q)\le\log\int p\cdot(p/q)=\log(1+\chi^2(p\|q))$,
and $\log(1+y)\le y$.
For \ref{fact:kl-product}, the density ratio is $dp/dq=\prod_t(p_t/q_t)$ (the common
factor $r$ cancels, whether or not it has a density), so
$\log(dp/dq)=\sum_t\log(p_t/q_t)$; integrating under $p$, whose $t$-th marginal is
$p_t$, gives $\sum_t\int p_t\log(p_t/q_t)$.
For \ref{fact:tv-test}, $\E_p\Xi-\E_q\Xi=\int(\Xi-K/2)(p-q)$ because $\int(p-q)=0$,
and $|\Xi-K/2|\le K/2$, so the absolute value is at most
$\frac K2\int|p-q|=K\,\TV(p,q)$.
\end{proof}

\subsubsection{Gaussian-mixture bounds}

\begin{lemma}[Gaussian identity and mixture comparison]
\label{lem:gauss-id}\label{lem:chi2}
For all real $a,b$,
\begin{equation}
\label{eq:gaussian-inner}
\int_{\mathbb R}\frac{\phi_a(x)\phi_b(x)}{\phi_0(x)}\,dx=e^{ab/\sigma^2}.
\end{equation}
Let $P=\sum_i w_i\delta_{s a_i}$ and $P'=\sum_i w_i'\delta_{s a_i}$ be probability
laws on a common finite node set, write $d_i=w_i-w_i'$ and
$M_k=\sum_i d_i a_i^k$, and suppose $P'*\phi\ge \kappa\phi_0$ pointwise for $\kappa>0$.
Then
\begin{equation}
\label{eq:mixture-series}
\KL(P*\phi\|P'*\phi)
\le\frac1\kappa\sum_{k=0}^{\infty}\frac{(s^2/\sigma^2)^k}{k!}M_k^2.
\end{equation}
\end{lemma}
\begin{proof}
By definition $\phi_a(x)=(2\pi\sigma^2)^{-1/2}\exp(-(x-a)^2/(2\sigma^2))$, so
\begin{equation}
\label{eq:gid-ratio}
\frac{\phi_a(x)\phi_b(x)}{\phi_0(x)}
=\frac1{\sqrt{2\pi\sigma^2}}
\exp\left(-\frac{(x-a)^2+(x-b)^2-x^2}{2\sigma^2}\right).
\end{equation}
Expanding the squares and completing the square in $x$,
\begin{equation}
\label{eq:gid-square}
(x-a)^2+(x-b)^2-x^2=x^2-2(a+b)x+a^2+b^2
=\big(x-(a+b)\big)^2-2ab.
\end{equation}
Substituting \eqref{eq:gid-square} into \eqref{eq:gid-ratio} gives
\begin{equation}
\label{eq:gid-pointwise}
\frac{\phi_a(x)\phi_b(x)}{\phi_0(x)}=e^{ab/\sigma^2}\phi_{a+b}(x).
\end{equation}
Since $\phi_{a+b}$ is a probability density, integrating \eqref{eq:gid-pointwise}
over $x$ proves \eqref{eq:gaussian-inner}.

Write $p=P*\phi$ and $q=P'*\phi$. If $\zeta\sim P$ is independent of
$\varepsilon\sim N(0,\sigma^2)$, then conditionally on $\zeta=sa_i$ the sum
$\zeta+\varepsilon$ has density $\phi_{sa_i}$; hence
\begin{equation}
\label{eq:gid-mixtures}
p=\sum_i w_i\phi_{sa_i},\qquad q=\sum_i w_i'\phi_{sa_i},\qquad
p-q=\sum_i d_i\phi_{sa_i}.
\end{equation}
By Lemma~\ref{lem:div-facts}\ref{fact:kl-chi} and the assumption $q\ge \kappa\phi_0$ (so $1/q\le1/(\kappa\phi_0)$),
\begin{equation}
\label{eq:gid-chi}
\KL(p\|q)\le\chi^2(p\|q)=\int\frac{(p-q)^2}{q}
\le\frac1\kappa\int\frac{(p-q)^2}{\phi_0}.
\end{equation}

Squaring the finite sum in \eqref{eq:gid-mixtures} and applying
\eqref{eq:gaussian-inner} with $a=sa_i$ and $b=sa_j$ to each term,
\begin{equation}
\label{eq:gid-double-sum}
\int\frac{(p-q)^2}{\phi_0}
=\sum_{i,j}d_id_j\int\frac{\phi_{sa_i}\phi_{sa_j}}{\phi_0}
=\sum_{i,j}d_id_j\,e^{s^2a_ia_j/\sigma^2}.
\end{equation}

Put $\omega=s^2/\sigma^2$ and expand $e^{\omega a_ia_j}=\sum_{k\ge0}\omega^ka_i^ka_j^k/k!$.
The sums over $i,j$ are finite and each exponential series converges absolutely, so the
order of summation may be exchanged:
\begin{equation}
\label{eq:gid-series}
\sum_{i,j}d_id_j\,e^{\omega a_ia_j}
=\sum_{k\ge0}\frac{\omega^k}{k!}\Big(\sum_id_ia_i^k\Big)\Big(\sum_jd_ja_j^k\Big)
=\sum_{k\ge0}\frac{\omega^k}{k!}M_k^2.
\end{equation}
Combining \eqref{eq:gid-chi}, \eqref{eq:gid-double-sum}, and \eqref{eq:gid-series}
proves \eqref{eq:mixture-series}.
\end{proof}

\begin{lemma}[Smoothing of symmetric moment-matched laws]
\label{lem:klcheb}
Let $P,P'$ be finite symmetric probability laws on $[-s,s]$, with equal moments of
orders $0,1,\ldots,2\nu-1$. If $s\le\sigma/2$, then
\begin{equation}
\label{eq:matched-kl}
\KL(P*\phi\|P'*\phi)\le\frac{6}{(2\nu)!}(s/\sigma)^{4\nu}.
\end{equation}
\end{lemma}
\begin{proof}
If $s=0$, then $P=P'=\delta_0$ and both sides vanish, so assume $s>0$.

Let $\{sa_i\}$ be the union of the (finite) supports of $P$ and $P'$, so that
$a_i\in[-1,1]$ and $P=\sum_iw_i\delta_{sa_i}$, $P'=\sum_iw_i'\delta_{sa_i}$, where
some weights may be zero. This is the setting of Lemma~\ref{lem:chi2}.

For every real $b$,
\begin{equation}
\label{eq:kc-shift}
\phi_b(x)=\phi_0(x)\exp\left(\frac{bx}{\sigma^2}-\frac{b^2}{2\sigma^2}\right),
\qquad\text{so}\qquad
\frac{\phi_b(x)+\phi_{-b}(x)}2
=\phi_0(x)e^{-b^2/(2\sigma^2)}\cosh\!\left(\frac{bx}{\sigma^2}\right).
\end{equation}
Since $\cosh\ge1$, and since $|b|\le s\le\sigma/2$ gives $b^2/(2\sigma^2)\le1/8$,
\begin{equation}
\label{eq:kc-pair}
\frac{\phi_b+\phi_{-b}}2\ge e^{-1/8}\phi_0\qquad(|b|\le s).
\end{equation}
Because $P'$ is symmetric, $\zeta'\sim P'$ and $-\zeta'$ have the same law, so
$P'*\phi=\E\phi_{\zeta'}=\E\phi_{-\zeta'}=\E\big[(\phi_{\zeta'}+\phi_{-\zeta'})/2\big]$.
Applying \eqref{eq:kc-pair} inside the expectation gives
\begin{equation}
\label{eq:kc-lower}
P'*\phi\ge e^{-1/8}\phi_0.
\end{equation}
Thus Lemma~\ref{lem:chi2} applies with $\kappa=e^{-1/8}$.

For each $k\ge0$,
\begin{equation}
\label{eq:kc-moment}
\E_P\zeta^k-\E_{P'}\zeta^k=\sum_i(w_i-w_i')(sa_i)^k=s^kM_k.
\end{equation}
The moments agree for $k\le2\nu-1$, so $M_k=0$ for $0\le k\le2\nu-1$. For every $k$,
since $|a_i|\le1$,
\begin{equation}
\label{eq:kc-Mk}
|M_k|\le\sum_i|d_i|\,|a_i|^k\le\sum_i(w_i+w_i')=2,
\qquad\text{so}\qquad M_k^2\le4.
\end{equation}

Let $\omega=s^2/\sigma^2\le1/4$. Inserting $\kappa=e^{-1/8}$, $M_k=0$ for $k<2\nu$, and
\eqref{eq:kc-Mk} into \eqref{eq:mixture-series},
\begin{equation}
\label{eq:kc-tail}
\KL(P*\phi\|P'*\phi)\le4e^{1/8}\sum_{k\ge2\nu}\frac{\omega^k}{k!}.
\end{equation}
For $k=2\nu+j$ with $j\ge0$, we have $k!=(2\nu)!\,(2\nu+1)\cdots(2\nu+j)\ge(2\nu)!\,j!$. Hence
\begin{equation}
\label{eq:kc-exp-tail}
\sum_{k\ge2\nu}\frac{\omega^k}{k!}
\le\frac{\omega^{2\nu}}{(2\nu)!}\sum_{j\ge0}\frac{\omega^j}{j!}
=\frac{\omega^{2\nu}e^\omega}{(2\nu)!}.
\end{equation}
Combining \eqref{eq:kc-tail} and \eqref{eq:kc-exp-tail} with $\omega\le1/4$ gives
\begin{equation}
\KL(P*\phi\|P'*\phi)\le4e^{1/8+1/4}\frac{\omega^{2\nu}}{(2\nu)!}
\le6\,\frac{\omega^{2\nu}}{(2\nu)!},
\end{equation}
because $4e^{3/8}<5.83<6$. Finally $\omega^{2\nu}=(s/\sigma)^{4\nu}$, which is
\eqref{eq:matched-kl}.
\end{proof}

\subsubsection{A weighted payoff inequality and the online reduction}

\begin{lemma}[Payoff with endpoint probabilities]
\label{lem:payoff}
Let $P$ be supported in $[-a,a]$ and $P'$ in $[-b,b]$, with
$0\le a<b$, $w_a=P(\{a\})>0$, and $w_b=P'(\{b\})>0$.
Define $V_P(z)=\E_{\zeta\sim P}[z+\zeta]_+$, $V_{P'}$ analogously, and
$\Delta=b-a$. Then, for every $z\in\mathbb R$,
\begin{equation}
\label{eq:weighted-payoff}
-a-z+\frac{V_P(z)}{w_a}
+\min\left\{\frac{V_{P'}(z)}{w_b},\Delta\right\}\ge\Delta.
\end{equation}
\end{lemma}
\begin{proof}
The integrand $[z+\zeta]_+$ is nonnegative, so dropping every atom except the one at
$\zeta=a$ can only decrease $V_P(z)$; similarly for $V_{P'}$ and the atom at $b$:
\begin{equation}
\label{eq:pay-atoms}
V_P(z)\ge w_a[z+a]_+,\qquad V_{P'}(z)\ge w_b[z+b]_+.
\end{equation}
Dividing by $w_a$ and $w_b$, and using that $y\mapsto\min\{y,\Delta\}$ is
nondecreasing, the left side of \eqref{eq:weighted-payoff} is at least
\begin{equation}
\label{eq:pay-reduced}
-a-z+[z+a]_++\min\{[z+b]_+,\Delta\}
=[-a-z]_++\min\{[z+b]_+,\Delta\},
\end{equation}
where we used $[y]_+-y=[-y]_+$ with $y=z+a$.

We check three cases, using $\Delta=b-a>0$.
\begin{itemize}[leftmargin=1.6em,itemsep=2pt]
\item If $z\le-b$, then $[-a-z]_+=-a-z\ge b-a=\Delta$, and the second term in
\eqref{eq:pay-reduced} is nonnegative.
\item If $-b\le z\le-a$, then $[-a-z]_+=-a-z$ and $0\le z+b\le b-a=\Delta$, so
\eqref{eq:pay-reduced} equals $(-a-z)+(z+b)=\Delta$.
\item If $z\ge-a$, then $[-a-z]_+=0$ and $z+b\ge b-a=\Delta$, so
\eqref{eq:pay-reduced} equals $\min\{z+b,\Delta\}=\Delta$.
\end{itemize}
In every case \eqref{eq:pay-reduced} is at least $\Delta$, which proves
\eqref{eq:weighted-payoff}.
\end{proof}

\begin{lemma}[Online observations and bounded testing statistics]
\label{lem:reduction}
Let $\gamma_t(x)=x-\frac12+\zeta_t$ and $\ell_t(x)=-x$ on $[0,1]$, with
$\zeta_{1:T}$ drawn independently from $P$ before play, and with exact gradients and
independent Gaussian value noise. Suppose $P$ is supported in $[-a,a]$, $a\le1/2$.
Writing $y_t=x_t-1/2$, the following hold.
\begin{enumerate}[leftmargin=1.6em,itemsep=2pt]
\item The action $u_P=1/2-a$ is feasible for every realization, and
$\Reg_T(u_P)=\sum_t(-a-y_t)$.
\item The transformed inputs $z_t=\tilde v_t-y_t=\zeta_t+\varepsilon_t$ are i.i.d.\
with law $P*\phi$. For the learner's independent internal randomness $\varpi$, the action sequence
is the same measurable function of $(z_{1:T},\varpi)$ in both worlds, with $y_t$ depending
only on $(z_{1:t-1},\varpi)$.
\item $\E_P\ICV_T=\E_P\sum_t V_P(y_t)$.
\item For any statistic $\Xi(z_{1:T},\varpi)\in[0,K]$ and another law $P'$,
\[
|\E_P\Xi-\E_{P'}\Xi|
\le K\sqrt{T\KL(P*\phi\|P'*\phi)/2}.
\]
\end{enumerate}
\end{lemma}
\begin{proof}
For claim~(1): since $0\le a\le1/2$, the point $u_P=1/2-a$ lies in $[0,1]$. Because
$\zeta_t\le a$ on every realization,
$\gamma_t(u_P)=(1/2-a)-1/2+\zeta_t=\zeta_t-a\le0$, so $u_P$ is feasible on every round.
With $\ell_t(x)=-x$,
\begin{equation}
\Reg_T(u_P)=\sum_{t=1}^T\big(-x_t+u_P\big)
=\sum_{t=1}^T\Big(\tfrac12-a-x_t\Big)=\sum_{t=1}^T(-a-y_t).
\end{equation}

For claim~(2), note that the exact gradients are $p_t=\ell_t'=-1$ and $q_t=\gamma_t'=1$ on every
round of both worlds, so they carry no information, and the learner's history before
round $t$ is equivalent to $(\tilde v_1,\ldots,\tilde v_{t-1})$. Hence
$x_t=f_t(\tilde v_{1:t-1},\varpi)$ for fixed measurable maps $f_t$ (the strategy).
The observation is
\begin{equation}
\label{eq:red-obs}
\tilde v_t=\gamma_t(x_t)+\varepsilon_t=y_t+\zeta_t+\varepsilon_t=y_t+z_t .
\end{equation}
We show by induction on $t$ that $x_t=g_t(z_{1:t-1},\varpi)$ for measurable maps $g_t$ that
do not depend on the law of $\zeta$. For $t=1$, $x_1=f_1(\varpi)$. If the claim holds for all
$s<t$, then each $y_s=g_s(z_{1:s-1},\varpi)-1/2$ is a function of $(z_{1:t-1},\varpi)$, so by
\eqref{eq:red-obs} each $\tilde v_s=y_s+z_s$ is too, and therefore so is
$x_t=f_t(\tilde v_{1:t-1},\varpi)$. Finally, the pairs $(\zeta_t,\varepsilon_t)$ are
independent across $t$ and of $\varpi$, and $\zeta_t$ is independent of $\varepsilon_t$; hence
$z_t=\zeta_t+\varepsilon_t$ are i.i.d.\ with density $P*\phi$ and independent of $\varpi$.

For claim~(3), note that the true value is $v_t=\gamma_t(x_t)=y_t+\zeta_t$. By claim~(2), $y_t$ is a
function of $(z_{1:t-1},\varpi)$, while $\zeta_t$ is independent of $(z_{1:t-1},\varpi)$.
Conditioning on $(z_{1:t-1},\varpi)$ and integrating over $\zeta_t\sim P$ alone,
\begin{equation}
\label{eq:red-icv}
\E_P\big[[y_t+\zeta_t]_+\,\big|\,z_{1:t-1},\varpi\big]=V_P(y_t).
\end{equation}
Taking expectations and summing over $t$ gives
$\E_P\ICV_T=\sum_t\E_P[v_t]_+=\E_P\sum_tV_P(y_t)$.

For claim~(4): by claim~(2), the input $(z_{1:T},\varpi)$ has joint law
$(P*\phi)^{\otimes T}\otimes P_\varpi$ in world $P$ and $(P'*\phi)^{\otimes T}\otimes P_\varpi$
in world $P'$, where $P_\varpi$ is the law of $\varpi$. By Lemma~\ref{lem:div-facts}\ref{fact:kl-product}, their relative
entropy is $T\KL(P*\phi\|P'*\phi)$. Pinsker's inequality, Lemma~\ref{lem:div-facts}\ref{fact:pinsker}, bounds their
total variation by $\sqrt{T\KL(P*\phi\|P'*\phi)/2}$, and Lemma~\ref{lem:div-facts}\ref{fact:tv-test}, applied to
$\Xi\in[0,K]$, proves the claim.
\end{proof}

\begin{lemma}[Master hard-violation lower bound]
\label{prop:master-lower}
Under the endpoint conditions of Lemma~\ref{lem:payoff}, suppose $b\le1/2$ and
$T\KL(P*\phi\|P'*\phi)\le1/2$.
Every learner then satisfies
\begin{equation}
\label{eq:master-hard}
\max\{\E_P\Reg_T(u_P),\E_P\ICV_T,\E_{P'}\ICV_T\}
\ge\frac{(b-a)T}{2(1+w_a^{-1}+w_b^{-1})}.
\end{equation}
\end{lemma}
\begin{proof}
Set $\Delta=b-a$, $R=\E_P\Reg_T(u_P)$, $V=\E_P\ICV_T$, $V'=\E_{P'}\ICV_T$, and
\begin{equation}
\label{eq:master-Xi}
\Xi=\sum_{t=1}^T\min\{V_{P'}(y_t)/w_b,\Delta\}\in[0,T\Delta].
\end{equation}
By Lemma~\ref{lem:reduction}(2), $\Xi$ is a function of $(z_{1:T},\varpi)$. Both laws satisfy
the hypothesis of Lemma~\ref{lem:reduction}, since $a<b\le1/2$.

Apply Lemma~\ref{lem:payoff} with $z=y_t$ and sum over $t$:
\begin{equation}
\sum_{t=1}^T(-a-y_t)+\frac1{w_a}\sum_{t=1}^TV_P(y_t)+\Xi\ge T\Delta.
\end{equation}
Take expectations in world $P$. By Lemma~\ref{lem:reduction}(1) the first sum has
expectation $R$, and by Lemma~\ref{lem:reduction}(3) the second has expectation $V$:
\begin{equation}
\label{eq:master-world-P}
R+\frac V{w_a}+\E_P\Xi\ge T\Delta.
\end{equation}

To pass to world $P'$, apply Lemma~\ref{lem:reduction}(4) with $K=T\Delta$ and
$T\KL\le1/2$, which gives
$\E_P\Xi\le\E_{P'}\Xi+T\Delta\sqrt{1/4}=\E_{P'}\Xi+T\Delta/2$.
Since $\min\{y,\Delta\}\le y$, Lemma~\ref{lem:reduction}(3) applied to $P'$ gives
$\E_{P'}\Xi\le\E_{P'}\sum_tV_{P'}(y_t)/w_b=V'/w_b$. Therefore
$\E_P\Xi\le V'/w_b+T\Delta/2$, and \eqref{eq:master-world-P} becomes
\begin{equation}
\label{eq:master-sum}
R+\frac V{w_a}+\frac{V'}{w_b}\ge\frac{T\Delta}2.
\end{equation}

Let $\mathsf m=\max\{R,V,V'\}$. Then $R\le\mathsf m$,
$V/w_a\le\mathsf m/w_a$, and $V'/w_b\le\mathsf m/w_b$; this holds even if $R<0$. Thus
the left side of \eqref{eq:master-sum} is at most $(1+w_a^{-1}+w_b^{-1})\mathsf m$, and
dividing proves \eqref{eq:master-hard}.
\end{proof}

\subsubsection{Explicit moment-matched laws with substantial endpoint mass}
\label{app:lower-constructions}

We use the elementary Chebyshev identities
$\mathrm T_m(\cos\theta)=\cos(m\theta)$ and
$\mathrm T_m(\cosh t)=\cosh(mt)$, together with Lagrange interpolation
\citep{trefethen2013approximation}.

For $\nu=1$, use $P_1=\delta_0$ and
$P_1'=\frac12\delta_{-s}+\frac12\delta_s$.
For $\nu\ge2$, set
\begin{equation}
\label{eq:cheb-nodes}
m=2\nu-2,\quad c=\cosh(1/m),\quad a=s/c,\quad b=s,\quad
\xi_j=\cos(j\pi/m)\quad(0\le j\le m).
\end{equation}
Define the Lagrange coefficients at the point $c>1$ by
\begin{equation}
\label{eq:cheb-weights}
\lambda_j=\prod_{\substack{0\le k\le m\\k\ne j}}
\frac{c-\xi_k}{\xi_j-\xi_k},\qquad
\beta_j=\frac{\lambda_j+\lambda_{m-j}}2,\qquad
\Psi=\frac{1+\cosh1}{2}.
\end{equation}
The probability laws are
\begin{align}
P_\nu&=\frac1\Psi\sum_{\substack{0\le j\le m\\j\ \mathrm{even}}}
\beta_j\delta_{a\xi_j},
\label{eq:cheb-P}\\*
P_\nu'&=\frac1\Psi\left(\frac{\delta_s+\delta_{-s}}2+
\sum_{\substack{0\le j\le m\\j\ \mathrm{odd}}}(-\beta_j)\delta_{a\xi_j}\right).
\label{eq:cheb-Q}
\end{align}

\begin{lemma}[Validity, moments, endpoints, and gap]
\label{lem:cheb-endpoints}
The laws $P_\nu,P_\nu'$ above are symmetric probabilities with identical moments through
order $2\nu-1$. Their maximum support points $a_\nu,b_\nu$ satisfy
\begin{equation}
\label{eq:cheb-gap}
b_\nu-a_\nu\ge\frac{s}{16\nu^2},\qquad
P_\nu(\{a_\nu\})\ge\frac13,\quad P_\nu'(\{b_\nu\})\ge\frac13.
\end{equation}
For $\nu\ge2$, $a_\nu=a$ and $b_\nu=s$; for $\nu=1$, $a_1=0$ and $b_1=s$.
\end{lemma}
\begin{proof}
If $\nu=1$, then $P_1=\delta_0$ and $P_1'$ are symmetric probability laws. Their
moments of order $0$ (both $1$) and order $1$ (both $0$) agree. The top support points
are $a_1=0$ and $b_1=s$, so $b_1-a_1=s\ge s/16$, and
$P_1(\{0\})=1$, $P_1'(\{s\})=1/2$ are at least $1/3$.

Assume from now on that $\nu\ge2$, so $m=2\nu-2\ge2$ is even.

The nodes $\xi_0=1>\xi_1>\cdots>\xi_m=-1$ are distinct, because $\cos$ is strictly
decreasing on $[0,\pi]$. The Lagrange basis polynomial
$\mathcal L_j(x)=\prod_{k\ne j}(x-\xi_k)/(\xi_j-\xi_k)$ has degree $m$ and satisfies
$\mathcal L_j(\xi_i)=\Ind\{i=j\}$. For a polynomial $p$ of degree at most $m$, the polynomial
$p-\sum_jp(\xi_j)\mathcal L_j$ has degree at most $m$ and vanishes at $m+1$ distinct points,
so it is zero. Evaluating at $x=c$, and recalling $\lambda_j=\mathcal L_j(c)$,
\begin{equation}
\label{eq:cheb-interp}
\sum_{j=0}^m\lambda_j\,p(\xi_j)=p(c)\qquad(\deg p\le m).
\end{equation}
We also use $\xi_{m-j}=\cos(\pi-j\pi/m)=-\xi_j$.

In the product defining $\lambda_j$, every numerator $c-\xi_k$ is positive because
$c>1\ge\xi_k$. The denominator $\xi_j-\xi_k$ is negative exactly when $k<j$, which
happens for $j$ values of $k$. Hence $\operatorname{sign}(\lambda_j)=(-1)^j$. Since $m$
is even, $j$ and $m-j$ have the same parity, so $\lambda_j$ and $\lambda_{m-j}$ have the
same sign. Consequently
\begin{equation}
\label{eq:cheb-beta-sign}
\operatorname{sign}(\beta_j)=(-1)^j,\qquad
|\beta_j|=\frac{|\lambda_j|+|\lambda_{m-j}|}2,\qquad\beta_{m-j}=\beta_j.
\end{equation}

Apply \eqref{eq:cheb-interp} to $p\equiv1$ and to the Chebyshev polynomial
$p=\mathrm T_m$, which has degree $m$. Since
$\mathrm T_m(\xi_j)=\cos(m\cdot j\pi/m)=(-1)^j$ and
$\mathrm T_m(c)=\mathrm T_m(\cosh(1/m))=\cosh1$,
\begin{equation}
\label{eq:cheb-lambda-sums}
\sum_{j=0}^m\lambda_j=1,\qquad
\sum_{j=0}^m|\lambda_j|=\sum_{j=0}^m(-1)^j\lambda_j=\cosh1,
\end{equation}
where the middle equality uses $\operatorname{sign}(\lambda_j)=(-1)^j$. Summing the definition of $\beta_j$ and the
identity for $|\beta_j|$ in \eqref{eq:cheb-beta-sign} over $j$, and reindexing
$j\mapsto m-j$ in the second halves,
\begin{equation}
\label{eq:cheb-beta-sums}
\sum_j\beta_j=\sum_j\lambda_j=1,\qquad
\sum_j|\beta_j|=\sum_j|\lambda_j|=\cosh1.
\end{equation}
By \eqref{eq:cheb-beta-sign}, $\sum_j|\beta_j|=\sum_{j\,\mathrm{even}}\beta_j
-\sum_{j\,\mathrm{odd}}\beta_j$. Adding and subtracting the two identities in
\eqref{eq:cheb-beta-sums},
\begin{equation}
\label{eq:cheb-pos-neg}
\sum_{j\,\mathrm{even}}\beta_j=\frac{1+\cosh1}2=\Psi,\qquad
\sum_{j\,\mathrm{odd}}(-\beta_j)=\frac{\cosh1-1}2=\Psi-1.
\end{equation}

In \eqref{eq:cheb-P}, every coefficient $\beta_j/\Psi$ with $j$ even is positive, and by
\eqref{eq:cheb-pos-neg} they sum to $\Psi/\Psi=1$. In \eqref{eq:cheb-Q}, every coefficient
$-\beta_j/\Psi$ with $j$ odd is positive, and the total mass is
$(1+(\Psi-1))/\Psi=1$. So both are probability laws. Their atoms are distinct, since
$|a\xi_j|\le a<ac=s$. They are symmetric because $\xi_{m-j}=-\xi_j$, $m-j$ has the
parity of $j$, and $\beta_{m-j}=\beta_j$; the extra atoms $\pm s$ in $P_\nu'$ have equal
mass.

Let $p$ have degree at most $m$. Applying \eqref{eq:cheb-interp} to the polynomials
$x\mapsto p(ax)$ and $x\mapsto p(-ax)$, both of degree at most $m$, and recalling
$ac=s$,
\begin{equation}
\sum_j\lambda_jp(a\xi_j)=p(s),\qquad
\sum_j\lambda_jp(-a\xi_j)=\sum_j\lambda_jp(a\xi_{m-j})
=\sum_j\lambda_{m-j}p(a\xi_j)=p(-s).
\end{equation}
Averaging the two identities,
\begin{equation}
\label{eq:cheb-moment-identity}
\sum_{j=0}^m\beta_jp(a\xi_j)=\frac{p(s)+p(-s)}2.
\end{equation}
Split the left side into even and odd $j$ and move the odd part to the right:
\begin{equation}
\sum_{j\,\mathrm{even}}\beta_jp(a\xi_j)
=\frac{p(s)+p(-s)}2+\sum_{j\,\mathrm{odd}}(-\beta_j)p(a\xi_j).
\end{equation}
After division by $\Psi$, the left side is $\E_{P_\nu}p(\zeta)$ and the right side is
$\E_{P_\nu'}p(\zeta)$. Taking $p(x)=x^k$ shows that the moments agree for
$0\le k\le m=2\nu-2$. Both laws are symmetric, so both moments of the odd order $2\nu-1$
vanish. Hence the moments agree through order $2\nu-1$.

All atoms of $P_\nu$ lie in $[-a,a]$ and the largest is $a\xi_0=a$, so $a_\nu=a$.
The odd-indexed nodes satisfy $|\xi_j|<1$, so the largest atom of $P_\nu'$ is $s$ and
$b_\nu=s$. For the masses, every factor in
$\lambda_0=\prod_{k=1}^m(c-\xi_k)/(1-\xi_k)$ exceeds one because $c>1$ and
$1-\xi_k>0$; hence $\lambda_0>1$. Also $\lambda_m>0$, because $\operatorname{sign}(\lambda_m)=(-1)^m$ and $m$ is even. Therefore
$\beta_0=(\lambda_0+\lambda_m)/2>1/2$ and
\begin{equation}
P_\nu(\{a\})=\frac{\beta_0}\Psi>\frac1{2\Psi}=\frac1{1+\cosh1}>\frac13,\qquad
P_\nu'(\{s\})=\frac1{2\Psi}=\frac1{1+\cosh1}>\frac13,
\end{equation}
using $\cosh1<1.55<2$.

For the gap, the power series gives $\cosh t-1=\sum_{k\ge1}t^{2k}/(2k)!\ge t^2/2$. With $t=1/m\le1/2$,
also $c=\cosh(1/m)\le\cosh(1/2)<2$. Therefore
\begin{equation}
s-a=s\Big(1-\frac1c\Big)=s\,\frac{c-1}{c}
\ge s\,\frac{1/(2m^2)}{2}=\frac{s}{4m^2}
\ge\frac{s}{16\nu^2},
\end{equation}
where the last step uses $m=2\nu-2<2\nu$.
\end{proof}

\subsection{\texorpdfstring{Proof of Theorem~\ref{thm:lower-n}}{Proof of Theorem 3.2}}
\label{app:lower-n-proof}

\begin{proof}[Proof of Theorem~\ref{thm:lower-n}]
Let $s$ be given by \eqref{eq:lower-scale}, and let $P_\nu,P_\nu'$ be the laws of
\eqref{eq:cheb-P}--\eqref{eq:cheb-Q}, or the two-point construction for $\nu=1$.
Draw $\zeta_t$ i.i.d.\ from the selected law and set $\gamma_t(x)=x-1/2+\zeta_t$ and
$\ell_t(x)=-x$ on $\Kset=[0,1]$. Both are $1$-Lipschitz with exact gradients, so
$D=G=\Gh=1$. Since $x-1/2\in[-1/2,1/2]$ and $|\zeta_t|\le s\le1/2$, the values lie in
$[-1,1]$. By Lemma~\ref{lem:reduction}(1), the comparator $u_\nu=1/2-a_\nu$ is feasible on
every realization under $P_\nu$, and $1/2-s$ is feasible on every realization under
$P_\nu'$, so both instances satisfy common feasibility.

By \eqref{eq:lower-scale}, $s\le\sigma/2$, and
$s\le\sigma((2\nu)!/(12T))^{1/(4\nu)}$, that is,
$(s/\sigma)^{4\nu}\le(2\nu)!/(12T)$. Lemma~\ref{lem:cheb-endpoints} shows that $P_\nu,P_\nu'$
satisfy the hypotheses of Lemma~\ref{lem:klcheb}, so
\begin{equation}
\label{eq:hard-kl}
T\KL(P_\nu*\phi\|P_\nu'*\phi)\le\frac{6T}{(2\nu)!}\Big(\frac s\sigma\Big)^{4\nu}
\le\frac{6T}{(2\nu)!}\cdot\frac{(2\nu)!}{12T}=\frac12.
\end{equation}

By Lemma~\ref{lem:cheb-endpoints}, $P_\nu$ is supported in $[-a_\nu,a_\nu]$ and $P_\nu'$ in
$[-b_\nu,b_\nu]$ with $a_\nu<b_\nu=s\le1/2$, $w_a,w_b\ge1/3$, and $b_\nu-a_\nu\ge s/(16\nu^2)$.
Thus $1+w_a^{-1}+w_b^{-1}\le7$, and Lemma~\ref{prop:master-lower} with
\eqref{eq:hard-kl} yields
\begin{equation}
\label{eq:hard-finite}
\max\{\E_{P_\nu}\Reg_T(u_\nu),\E_{P_\nu}\ICV_T,\E_{P_\nu'}\ICV_T\}
\ge\frac{T}{2\cdot7}\cdot\frac{s}{16\nu^2}=\frac{sT}{224\nu^2}.
\end{equation}
This is \eqref{eq:lower-n-finite}.

Now fix $\nu$ and $\sigma$ and let $T$ grow. The first two terms in \eqref{eq:lower-scale} do not depend on $T$, while the third
decreases to zero as $T\to\infty$; so for all sufficiently large $T$,
$s=\sigma((2\nu)!/(12T))^{1/(4\nu)}$. Since $(2\nu)!\ge2$, we have $(2\nu)!/12\ge1/6$, and
$(1/6)^{1/(4\nu)}\ge(1/6)^{1/4}>1/2$. Hence
$s\ge\frac\sigma2T^{-1/(4\nu)}$, and \eqref{eq:hard-finite} gives
$\sigma T^{1-1/(4\nu)}/(448\nu^2)$.

For the growing-order bound, let $\nu(T)$ be the smallest positive integer with $(2\nu)!16^\nu\ge12T$. For this $\nu$,
$(2\nu)!/(12T)\ge16^{-\nu}$, so the third term in \eqref{eq:lower-scale} is at least
$\sigma\cdot16^{-1/4}=\sigma/2$. Then $s=\min\{1/2,\sigma/2\}=\min\{\sigma,1\}/2$, and
\eqref{eq:hard-finite} becomes
\begin{equation}
\label{eq:hard-growing-explicit}
\frac{\min\{\sigma,1\}T}{448\nu(T)^2}.
\end{equation}

To bound $\nu(T)$, assume first $T\ge16$ and write $\tau=\log T$, so $\tau>\log12$ and
$\log\tau>0$. From $e^k\ge k^k/k!$ we get $k!\ge(k/e)^k$, hence
\begin{equation}
\label{eq:hard-stirling}
\log\big((2\nu)!16^\nu\big)\ge2\nu\log\frac{2\nu}e .
\end{equation}
The right side is increasing in $\nu\ge1$ (its derivative is $2\log(2\nu)>0$).
Let $\nu_0=\lceil2\tau/\log\tau\rceil$. The function
$\psi(\tau)=4\sqrt{\tau}-e\log\tau$ has its minimum on $(0,\infty)$ at
$\tau=e^2/4$, where it is positive; hence $4\tau/(e\log\tau)\ge\sqrt{\tau}$.
Using $\nu_0\ge2\tau/\log\tau$ and monotonicity in \eqref{eq:hard-stirling},
\begin{equation}
\log\big((2\nu_0)!16^{\nu_0}\big)
\ge\frac{4\tau}{\log\tau}\log\frac{4\tau}{e\log\tau}
\ge\frac{4\tau}{\log\tau}\cdot\frac{\log\tau}2
=2\tau\ge\tau+\log12=\log(12T).
\end{equation}
Thus $\nu(T)\le \nu_0\le2\tau/\log\tau+1\le3\tau/\log\tau$, where the last
step uses $\tau\ge\log\tau$. Substituting into \eqref{eq:hard-growing-explicit},
for $T\ge16$ the bound is at least
\begin{equation}
\frac{\min\{\sigma,1\}T}{4032}\left(\frac{\log\log T}{\log T}\right)^2,
\end{equation}
which proves \eqref{eq:lower-growing} with absolute constants. For $3\le T<16$ we
have $\nu(T)\le2$, because $4!\cdot16^2=6144\ge12\cdot16$; then
\eqref{eq:hard-growing-explicit} is at least $\min\{\sigma,1\}T/1792$, which is at least
$\frac{\min\{\sigma,1\}T}{4032}\big(\frac{\log\log T}{\log T}\big)^2$ because
$\log\log T/\log T<1$ there. The same order also
follows from Stirling's formula \citep{robbins1955remark}.
\end{proof}

\begin{remark}[Randomized instances and deterministic extraction]
\label{rem:lower-quant}
The instance distribution is oblivious, since the entire sequence of shifts
$\zeta_1,\ldots,\zeta_T$, and hence of constraint functions, is drawn before play, and the comparator $u_\nu$ is fixed and feasible for every realization in its world,
so it is an admissible special case of the comparators allowed by the upper
bound. For whichever
expected quantity attains the maximum in \eqref{eq:lower-n-finite}, some realization of
$\zeta_{1:T}$ has conditional expectation at least as large, which gives a
deterministic oblivious instance for each learner.
The learner may be told the loss and the exact constraint slope, since both are the
same in the two worlds.
\end{remark}

\begin{remark}[Scope of the lower bounds]
\label{rem:lower-scope}
The Gaussian lower bounds (Theorem~\ref{thm:lower-n} and
Proposition~\ref{prop:lower-wcv-log}) remain valid if the learner is told the loss, the action $0$ (feasible
in both instances), and exact derivatives at arbitrary points: these are the same in the
two instances ($\ell_t'=-1$ and $\gamma_t'=1$ in the hard-violation construction,
$\gamma_t'=\rho$ in the budget construction, with vanishing higher derivatives), so they
carry no information about the missing offset, and the bounds hold for smooth affine
functions. Truncating the Gaussian does not by itself give a bounded-noise version of
Theorem~\ref{thm:lower-n}, because with bounded noise the support of the observations can
reveal the largest value of $\zeta_t$; Proposition~\ref{prop:lower-linear} uses bounded
noise but lets the noise law differ between its two instances. None of the lower bounds
applies when the value observation is exact in every instance.
\end{remark}

\subsection{\texorpdfstring{Proof of Proposition~\ref{prop:lower-wcv-log}}{Proof of Proposition 3.3}}
\label{app:lower-budget-log}

\begin{proof}[Proof of Proposition~\ref{prop:lower-wcv-log}]
Use \eqref{eq:budget-pair} on $\Kset=[0,1]$ with $\ell_t(x)=-x$, exact gradients
$p_t=-1$, $q_t=\rho$, and i.i.d.\ $N(0,\sigma^2)$ value noise independent of the
learner's internal randomness $\varpi$. Since $0<\rho\le1$, the instances satisfy $D=G=\Gh=1$ and
$|\gamma_t|\le1$. In world $A$, $\gamma^A_t(x)=\rho(x-1)\le0$ for all $x\in[0,1]$, so
$u_A=1$ is feasible and
\begin{equation}
\label{eq:budget-regA}
\Reg_T(1)=\sum_{t=1}^T(1-x_t)\ge0.
\end{equation}
In world $B$, $v_t=\rho x_t\ge0$ for every action, so the three violation metrics
coincide:
\begin{equation}
\label{eq:budget-metricsB}
\WCV_T=\CCV_T=\ICV_T=\rho\sum_{t=1}^Tx_t.
\end{equation}

Let $Y_t=\tilde v_t-\rho x_t$. In world $A$, $Y_t=\rho(x_t-1)+\varepsilon_t-\rho x_t
=-\rho+\varepsilon_t$; in world $B$, $Y_t=\varepsilon_t$. So $Y_1,\ldots,Y_T$ are
i.i.d.\ $N(-\rho,\sigma^2)$ under $A$ and i.i.d.\ $N(0,\sigma^2)$ under $B$, independent
of $\varpi$. The gradients are identical in both worlds, and $\tilde v_t=Y_t+\rho x_t$.
Exactly as in Lemma~\ref{lem:reduction}(2), induction on $t$ shows that
$x_t=g_t(Y_{1:t-1},\varpi)$ for the same measurable maps $g_t$ in both worlds.

For one coordinate, the density ratio of $N(0,\sigma^2)$ to $N(-\rho,\sigma^2)$ at $y$ is
\begin{equation}
\exp\left(\frac{(y+\rho)^2-y^2}{2\sigma^2}\right)
=\exp\left(\frac{\rho y}{\sigma^2}+\frac{\rho^2}{2\sigma^2}\right).
\end{equation}
The internal randomness has the same law in both worlds, so the likelihood ratio of the inputs
$(Y_{1:T},\varpi)$ is the product
\begin{equation}
\label{eq:budget-J}
J=\frac{dP_B}{dP_A}
=\exp\left(\frac\rho{\sigma^2}\sum_{t=1}^TY_t+\frac{T\rho^2}{2\sigma^2}\right).
\end{equation}
Under $A$ the $Y_t$ are independent, and the Gaussian moment-generating function
$\E e^{\lambda Y}=e^{-\lambda\rho+\lambda^2\sigma^2/2}$ for $Y\sim N(-\rho,\sigma^2)$
with $\lambda=2\rho/\sigma^2$ gives $\E_Ae^{2\rho Y_t/\sigma^2}=1$. Hence
\begin{equation}
\label{eq:budget-J2}
\E_AJ^2=e^{T\rho^2/\sigma^2}\prod_{t=1}^T\E_Ae^{2\rho Y_t/\sigma^2}
=e^{T\rho^2/\sigma^2}.
\end{equation}

Let $Z=T^{-1}\sum_t(1-x_t)\in[0,1]$, a function of $(Y_{1:T},\varpi)$. Put
$\mathcal R_A=\E_A\Reg_T(1)$ and $\mathcal V_B=\E_B\WCV_T$. By \eqref{eq:budget-regA} and
\eqref{eq:budget-metricsB},
\begin{equation}
\label{eq:budget-Z}
\E_AZ=\frac{\mathcal R_A}T,\qquad \E_BZ=1-\frac{\E_B\sum_tx_t}T=1-\frac{\mathcal V_B}{\rho T}.
\end{equation}
Changing measure, then applying Cauchy--Schwarz, \eqref{eq:budget-J2}, and
$Z^2\le Z$ (valid since $Z\in[0,1]$),
\begin{equation}
\label{eq:budget-CS}
1-\frac{\mathcal V_B}{\rho T}=\E_BZ=\E_A[JZ]\le\sqrt{\E_AJ^2}\sqrt{\E_AZ^2}
\le e^{T\rho^2/(2\sigma^2)}\sqrt{\frac{\mathcal R_A}T}.
\end{equation}
This is \eqref{eq:budget-testing}.

Let $0<\mathcal B_T\le T/2$ and set $\rho=2\mathcal B_T/T\in(0,1]$. Both worlds belong to
$\mathfrak C_\sigma^{\mathrm G}$, so the uniform guarantees give $\mathcal V_B\le\mathcal B_T=\rho T/2$ and
$\mathcal R_A\le R_T$. The first makes the left side of \eqref{eq:budget-CS} at least $1/2$.
Squaring \eqref{eq:budget-CS} and rearranging,
\begin{equation}
\label{eq:budget-final}
R_T\ge \mathcal R_A\ge\frac T4e^{-T\rho^2/\sigma^2}
=\frac T4\exp\left(-\frac{4\mathcal B_T^2}{\sigma^2T}\right),
\end{equation}
which is \eqref{eq:lower-wcv-log}.

If moreover $0<R_T<T/4$, then \eqref{eq:budget-final} gives
$\exp(4\mathcal B_T^2/(\sigma^2T))\ge T/(4R_T)>1$. Taking logarithms,
$\mathcal B_T^2\ge\frac{\sigma^2T}4\log\frac T{4R_T}$, and taking square roots proves
\eqref{eq:budget-inverse}.

By \eqref{eq:budget-metricsB}, $\CCV_T=\WCV_T$ in world $B$, the only world in which the
budget guarantee was used, so the proof applies verbatim to terminal violation. Every $\rho\in(0,1]$ gives an instance of $\mathfrak C_\sigma^{\mathrm G}$, because the
declared constants $D=G=\Gh=1$ do not depend on $\rho$.
\end{proof}

\subsection{\texorpdfstring{Proof of Corollary~\ref{cor:budget-joint}}{Proof of Corollary 3.4}}
\label{app:lower-budget-joint}

\begin{proof}[Proof of Corollary~\ref{cor:budget-joint}]
We first prove \eqref{eq:joint-budget-lower}. Fix $T\ge2$, so $\log T>0$. Set $\bar\sigma=\min\{\sigma,1\}$ and
\begin{equation}
\rho=\frac{\bar\sigma}2\sqrt{\frac{\log T}T},\qquad
h=\frac{\bar\sigma}4\sqrt{T\log T}=\frac{\rho T}2.
\end{equation}
Since $\log T\le T$, $\rho\in(0,1/2]$, so both worlds are admissible. If $\mathcal R_A\ge h$,
then \eqref{eq:joint-budget-lower} holds. Otherwise $\mathcal R_A<h$, and we bound the two
factors on the right of \eqref{eq:budget-testing}. Since $\bar\sigma\le\sigma$,
\begin{equation}
\frac{T\rho^2}{2\sigma^2}=\frac{\bar\sigma^2\log T}{8\sigma^2}\le\frac{\log T}8,
\qquad\text{so}\qquad e^{T\rho^2/(2\sigma^2)}\le T^{1/8}.
\end{equation}
Also $\sqrt{\mathcal R_A/T}<\sqrt{h/T}=\frac{\sqrt{\bar\sigma}}2(\log T)^{1/4}T^{-1/4}$.
Therefore, by \eqref{eq:budget-testing},
\begin{equation}
1-\frac{\mathcal V_B}{\rho T}
<T^{1/8}\cdot\frac{\sqrt{\bar\sigma}}2\frac{(\log T)^{1/4}}{T^{1/4}}
=\frac{\sqrt{\bar\sigma}}2\frac{(\log T)^{1/4}}{T^{1/8}}\le\frac12,
\end{equation}
using $\bar\sigma\le1$ and $\log T\le\sqrt T$ (the function $\log x-\sqrt x$ is negative
on $[1,\infty)$). Rearranging, $\mathcal V_B>\rho T/2=h$. In both cases the maximum in
\eqref{eq:joint-budget-lower} is at least $h$.

Next we show that square-root regret forces the logarithm. Suppose $R_T\le C\sqrt T$ for a constant $C>0$. If $\mathcal B_T=0$, then in world $B$
(with any admissible $\rho>0$), $\rho\E_B\sum_tx_t=0$ forces $x_t=0$ for all $t$, almost
surely under $B$. The event $\{x_t=0\ \forall t\}$ is determined by $(Y_{1:T},\varpi)$, and
the input laws in the two worlds are mutually absolutely continuous (Gaussians with the
same variance), so the event also has probability one under $A$. Then $\mathcal R_A=T$,
contradicting $R_T\le C\sqrt T$ once $T>C^2$. If $\mathcal B_T>T/2$, then
$\mathcal B_T$ already exceeds any fixed multiple of $\sigma\sqrt{T\log T}$ for large
$T$. In the remaining case $0<\mathcal B_T\le T/2$, \eqref{eq:budget-final} shows
$R_T>0$, and for $T>16C^2$ we have $R_T\le C\sqrt T<T/4$. Then
\eqref{eq:budget-inverse} and $\log(T/(4R_T))\ge\log(T/(4C\sqrt T))$ give
\begin{equation}
\mathcal B_T\ge\frac\sigma2
\sqrt{T\left(\tfrac12\log T-\log(4C)\right)},
\end{equation}
which is $\Omega(\sigma\sqrt{T\log T})$ for fixed $\sigma$.

Finally, we show that a square-root budget forces linear regret. Suppose $\mathcal B_T\le C_B\sqrt T$. Then $C_B\sqrt T$ is also a valid uniform budget
bound, and $0<C_B\sqrt T\le T/2$ whenever $T\ge4C_B^2$. Applying
\eqref{eq:lower-wcv-log} with $C_B\sqrt T$ in place of $\mathcal B_T$,
\begin{equation}
R_T\ge\frac T4\exp\left(-\frac{4C_B^2T}{\sigma^2T}\right)
=\frac14e^{-4C_B^2/\sigma^2}T,
\end{equation}
which is the last claim of the corollary.
\end{proof}

\begin{remark}[Vanishing noise]
\label{rem:vanishing-noise}
The lower bound of Corollary~\ref{cor:budget-joint} vanishes with $\sigma$, but the class
$\mathfrak C_\sigma$ also contains unconstrained online learning, whose regret is of order
$\sqrt T$ for every $\sigma$. Take $\Kset=[0,1]$, $\gamma_t(x)=-1$, and $\ell_t(x)=s_tx$, with
independent uniform signs $s_t$ drawn before play, exact gradients, and any admissible value
noise. Every action is feasible on every round, so every violation is zero. Since $x_t$ is
independent of $s_t$, $\E\Reg_T(U_T^*)=\E\max\{-S_T,0\}=\frac12\E|S_T|$ with
$S_T=\sum_ts_t$, and $\E S_T^2=T$, $\E S_T^4\le3T^2$, and H\"older's inequality give
$\E|S_T|\ge(\E S_T^2)^{3/2}/(\E S_T^4)^{1/2}\ge\sqrt{T/3}$. Together with
Corollary~\ref{cor:budget-joint}, for $V\in\{\CCV,\WCV\}$ and $T\ge2$,
\[
\mathfrak M_T(V)\ge\max\Big\{\frac{\sqrt T}{2\sqrt3},\ \frac{\min\{\sigma,1\}}4\sqrt{T\log T}\Big\}.
\]
The upper bound of Theorem~\ref{thm:exp} is of order $\sqrt{T\log T}$ for every $\sigma\le1$,
so for a noise level $\sigma=\sigma_T\to0$ the gap to this lower bound is at most of order
$\sqrt{\log T}$, rather than $1/\sigma_T$. Whether this factor is needed is not settled here;
the optimality statement holds for each fixed $\sigma>0$.
\end{remark}

\section{\texorpdfstring{Proof of Theorem~\ref{thm:exp}}{Proof of Theorem 4.1}}
\label{app:exp}

This appendix proves Theorem~\ref{thm:exp}. Throughout, $L=D\Gh$, $B_v=(L^2+\sigma^2)^{1/2}\ge L$,
$K_0=4+10L$, $\Lambda_T=\log((T+1)K_0T)$, and $0<\alpha\le(128B_v)^{-1}$ is fixed. All
quantities are those of Algorithm~\ref{alg:ledger-exp} run with this $\alpha$; the
prescribed value is $\alpha=\min\{(128B_v)^{-1},\sqrt{\Lambda_T/(17T)}/B_v\}$.
Comparators $U\in\Ucal$ are random actions that are almost surely feasible on every round
and may depend on the whole run; conditional expectations use the pre-noise field $\Fs_t$
and martingales use the post-feedback filtration $(\mathcal G_t)$.

\paragraph{Outline.}
Appendix~\ref{app:exp-aux1} proves three lemmas, Appendix~\ref{app:exp-master} combines
them into the master inequality \eqref{eq:exp-master}, Appendix~\ref{app:exp-aux2} uses
it to control the running maximum and the regularizers, and Appendix~\ref{app:exp-budget}
proves \eqref{eq:exp-wcv} and \eqref{eq:exp-rates}. Lemma~\ref{lem:exp-primal} is a version of online gradient
descent~\citep{zinkevich2003online} with predictable, variable step sizes; to allow a
comparator chosen in hindsight it adds an auxiliary sequence as in
\citet[proof of Lemma~3.1, Lemma~6.1]{nemirovski2009robust}. The weighted
cancellation \eqref{eq:exp-cancel} and the bound on the regularizer in
Lemma~\ref{lem:exp-reg-cost} adapt \citet[Alg.~1 and Lemmas~10--11]{yu2026withoutslater}.
Specific to this paper are the drift bound for the reset potential under dependent
finite-variance feedback (Lemma~\ref{lem:exp-drift}) and the transfer
\eqref{eq:exp-window-transfer} to the true expected maximum overspend, which pays for
both truncation and regularization. The remaining steps use exponential moments, Jensen's
inequality, and Doob's inequality.

\paragraph{Probabilistic tools.}
We repeatedly use the following properties of conditional expectation, for integrable
random variables $X$ and a $\sigma$-field $\mathcal H$:
(i)~\emph{tower}: $\E[\E[X\mid\mathcal H]]=\E X$, and
$\E[\E[X\mid\mathcal H]\mid\mathcal H']=\E[X\mid\mathcal H']$ for
$\mathcal H'\subseteq\mathcal H$;
(ii)~\emph{pull-out}: if $Y$ is $\mathcal H$-measurable and $XY$ is integrable (or
$X,Y\ge0$), then
$\E[XY\mid\mathcal H]=Y\E[X\mid\mathcal H]$;
(iii)~\emph{Jensen}: $\E\varphi(X)\le\varphi(\E X)$ for concave $\varphi$ (we use
$\varphi=\log$ and $\varphi=\sqrt{\cdot}$), and
$\|\E[X\mid\mathcal H]\|\le\E[\|X\|\mid\mathcal H]$ for vectors.
The learner's quantities $w_t,H_t,\eta_t,r_t$ (and $x_t$) are computed before the round-$t$
feedback, so they are $\mathcal G_{t-1}$-measurable and hence $\Fs_t$-measurable, and
can be pulled out of $\E[\,\cdot\mid\Fs_t]$.

\subsection{\texorpdfstring{Auxiliary lemmas for the master inequality~\eqref{eq:exp-master}}{Auxiliary lemmas for the master inequality (17)}}
\label{app:exp-aux1}

\subsubsection{Truncation, domination, and integrability}

For analysis only, define
\[
v_t^\circ=\max\{v_t,-L\},\qquad
\tilde v_t^\circ=v_t^\circ+\varepsilon_t,\qquad
\bar v_t^\circ=\operatorname{clip}(\tilde v_t^\circ,-1/\alpha,1/\alpha).
\]
These quantities are not queried or computed by the algorithm.

\begin{lemma}[Dominating observations and predictable regularizers]
\label{lem:exp-clip}
For every round,
\begin{align}
\bar v_t&\le\bar v_t^\circ,\qquad
\E[(\bar v_t^\circ)^2\mid\Fs_t]\le B_v^2,
\label{eq:exp-clip-second}\\
|\E[\bar v_t^\circ\mid\Fs_t]-v_t^\circ|
&\le\E[|\tilde v_t^\circ-\bar v_t^\circ|\mid\Fs_t]\le\alpha B_v^2,
\label{eq:exp-clip-bias}\\
\E[[\tilde v_t-\bar v_t]_+\mid\Fs_t]&\le\alpha B_v^2,
\label{eq:exp-positive-residual}\\
16\alpha B_v^2\le r_t&\le16\alpha B_v^2+8L\le\tfrac{65}{8}B_v.
\label{eq:exp-r-bound}
\end{align}
All terms whose expectations occur in the bounds below are integrable; no unconditional
moment of a negative constraint value or of its observation is needed.
\end{lemma}
\begin{proof}
Write $\clip(z)=\clip(z,-1/\alpha,1/\alpha)=\max\{-1/\alpha,\min\{z,1/\alpha\}\}$.

By Lemma~\ref{lem:anchor}, $v_t\le v_t^\circ$ and $|v_t^\circ|\le L$. Adding the same
error $\varepsilon_t$ gives $\tilde v_t=v_t+\varepsilon_t\le v_t^\circ+\varepsilon_t
=\tilde v_t^\circ$. The map $z\mapsto\clip(z)$ is nondecreasing, so
$\bar v_t=\clip(\tilde v_t)\le\clip(\tilde v_t^\circ)=\bar v_t^\circ$.

The value $v_t^\circ$ is a function of $v_t$, hence $\Fs_t$-measurable. Expanding the
square and using pull-out together with \eqref{eq:feedback} and
\eqref{eq:finite-variance},
\begin{equation}
\label{eq:clip-pf-second}
\E[(\tilde v_t^\circ)^2\mid\Fs_t]
=(v_t^\circ)^2+2v_t^\circ\E[\varepsilon_t\mid\Fs_t]+\E[\varepsilon_t^2\mid\Fs_t]
\le L^2+0+\sigma^2=B_v^2.
\end{equation}
Truncation moves a number toward zero, so $|\clip(z)|\le|z|$, and therefore
$\E[(\bar v_t^\circ)^2\mid\Fs_t]\le\E[(\tilde v_t^\circ)^2\mid\Fs_t]\le B_v^2$.
This proves \eqref{eq:exp-clip-second}.

For every real $z$, $|z-\clip(z)|=(|z|-1/\alpha)_+$. If $|z|\le1/\alpha$ this is
zero; otherwise $\alpha|z|>1$ and $|z|-1/\alpha\le|z|\le\alpha z^2$. Thus
\begin{equation}
\label{eq:clip-pf-error}
|z-\clip(z)|=(|z|-1/\alpha)_+\le\alpha z^2\qquad(z\in\R).
\end{equation}
Since $\E[\tilde v_t^\circ\mid\Fs_t]=v_t^\circ+\E[\varepsilon_t\mid\Fs_t]=v_t^\circ$,
conditional Jensen, \eqref{eq:clip-pf-error}, and \eqref{eq:clip-pf-second} give
\begin{equation}
|\E[\bar v_t^\circ\mid\Fs_t]-v_t^\circ|
=|\E[\bar v_t^\circ-\tilde v_t^\circ\mid\Fs_t]|
\le\E[|\tilde v_t^\circ-\bar v_t^\circ|\mid\Fs_t]
\le\alpha\E[(\tilde v_t^\circ)^2\mid\Fs_t]\le\alpha B_v^2,
\end{equation}
which is \eqref{eq:exp-clip-bias}.

If $\tilde v_t\le1/\alpha$, then $\bar v_t\ge\tilde v_t$ and $[\tilde v_t-\bar v_t]_+=0$.
If $\tilde v_t>1/\alpha$, then $\bar v_t=1/\alpha$. In both cases
$[\tilde v_t-\bar v_t]_+=(\tilde v_t-1/\alpha)_+$. Since $\tilde v_t\le\tilde v_t^\circ$ (shown
above), and by \eqref{eq:clip-pf-error},
\begin{equation}
[\tilde v_t-\bar v_t]_+=(\tilde v_t-1/\alpha)_+
\le(\tilde v_t^\circ-1/\alpha)_+
\le(|\tilde v_t^\circ|-1/\alpha)_+
\le\alpha(\tilde v_t^\circ)^2.
\end{equation}
Taking $\E[\,\cdot\mid\Fs_t]$ and using \eqref{eq:clip-pf-second} proves
\eqref{eq:exp-positive-residual}. 

Since $H_t=H_{t-1}+w_t^2\ge w_t^2$, we have
$\sqrt{H_t}+\sqrt{H_{t-1}}\ge\sqrt{H_t}\ge w_t\ge0$, so the adaptive part
$8Lw_t/(\sqrt{H_t}+\sqrt{H_{t-1}})$ of $r_t$ lies in $[0,8L]$. This gives the first two
inequalities in \eqref{eq:exp-r-bound}. For the last, use $\alpha B_v\le1/128$ and
$L\le B_v$: $16\alpha B_v^2+8L\le B_v/8+8B_v=\frac{65}8B_v$.

Finally, we record deterministic bounds and integrability. Since $r_t\ge0$ and $\bar v_t\le1/\alpha$, we have
$Q_t=[Q_{t-1}+\bar v_t-r_t]_+\le Q_{t-1}+1/\alpha$, so $0\le Q_t\le t/\alpha$. Hence
$0\le w_t\le\alpha(e^{t-1}-1)$, and $H_t$, $\Phi_\alpha(Q_t)$, $r_t$ are bounded by
deterministic constants for each $t\le T$. For the primal terms, let
$d_t=\tilde p_t+w_t\tilde q_t$ and $\mu_t=p_t+w_tq_t$. Then
$\E\|d_t\|^2\le2\Gh^2(1+\alpha^2(e^{t-1}-1)^2)<\infty$, by $\|a+b\|^2\le2\|a\|^2+2\|b\|^2$ and
\eqref{eq:finite-variance}, and $\|\mu_t\|\le\Gh(1+w_t)$ is bounded, because
$\|p_t\|,\|q_t\|\le\Gh$ by the argument in Lemma~\ref{lem:anchor}. Hence, for any random
$y,y'\in\Kset$, the products $\ip{d_t}{y-y'}$ and $\ip{d_t-\mu_t}{y-y'}$, bounded in
absolute value by $D\|d_t\|$ and $D(\|d_t\|+\|\mu_t\|)$, are integrable. Loss differences
satisfy $|\ell_t(x_t)-\ell_t(U)|\le G\|x_t-U\|\le GD\le L$, and Lemma~\ref{lem:anchor}
gives $v_t\le L$, so $0\le\WCV_n\le Ln$. The positive truncation residual is integrable, being at most
$\alpha(\tilde v_t^\circ)^2$.
\end{proof}

\subsubsection{Predictable primal comparison}

\begin{lemma}[Variable-step primal bound against a comparator chosen in hindsight]
\label{lem:exp-primal}
For every comparator $U\in\Ucal$ and $n\le T$,
\begin{equation}
\label{eq:exp-primal}
\E\Reg_n(U)+\E\sum_{t=1}^n w_tv_t^\circ\le4L\sqrt n+4L\E\sqrt{H_n}.
\end{equation}
\end{lemma}

\begin{proof}
Write $d_t=\tilde p_t+w_t\tilde q_t$, $\mu_t=p_t+w_tq_t$, $\xi_t=d_t-\mu_t$,
$S_t=t+H_t$ for $t\ge1$, and $S_0=\alpha^2=H_0$. Then $S_t-S_{t-1}=1+w_t^2>0$ for every
$t\ge1$, so $S_t$ is increasing and the step sizes $\eta_t=D/(2\Gh\sqrt{S_t})$ are
nonincreasing.

We begin with a pathwise projection inequality. Let $(y_t)_{t\ge1}$ be points of $\Kset$ and
$(g_t)_{t\ge1}$ vectors with $y_{t+1}=\Pi_\Kset(y_t-\eta_tg_t)$, and let $y\in\Kset$. The
Euclidean projection onto the convex set $\Kset$ is nonexpansive and fixes $y$, so
$\|\Pi_\Kset(z)-y\|\le\|z-y\|$ for all $z$. With $z=y_t-\eta_tg_t$,
\begin{equation}
\|y_{t+1}-y\|^2\le\|y_t-\eta_tg_t-y\|^2
=\|y_t-y\|^2-2\eta_t\ip{g_t}{y_t-y}+\eta_t^2\|g_t\|^2.
\end{equation}
Rearranging,
\begin{equation}
\label{eq:primal-pf-step}
\ip{g_t}{y_t-y}\le\frac{\|y_t-y\|^2-\|y_{t+1}-y\|^2}{2\eta_t}+\frac{\eta_t}2\|g_t\|^2.
\end{equation}
Sum \eqref{eq:primal-pf-step} over $t=1,\ldots,n$ and regroup the first terms by
$\|y_t-y\|^2$ (summation by parts):
\begin{equation}
\sum_{t=1}^n\frac{\|y_t-y\|^2-\|y_{t+1}-y\|^2}{2\eta_t}
=\frac{\|y_1-y\|^2}{2\eta_1}
+\sum_{t=2}^n\|y_t-y\|^2\Big(\frac1{2\eta_t}-\frac1{2\eta_{t-1}}\Big)
-\frac{\|y_{n+1}-y\|^2}{2\eta_n}.
\end{equation}
Each bracket is nonnegative because $\eta_t\le\eta_{t-1}$, each $\|y_t-y\|^2\le D^2$,
and the last term is nonpositive. Hence the sum is at most
$D^2/(2\eta_1)+D^2\sum_{t=2}^n(1/(2\eta_t)-1/(2\eta_{t-1}))=D^2/(2\eta_n)$, and
\begin{equation}
\label{eq:exp-ogd}
\sum_{t=1}^n\ip{g_t}{y_t-y}
\le\frac{D^2}{2\eta_n}+\frac12\sum_{t=1}^n\eta_t\|g_t\|^2
\qquad(y\in\Kset).
\end{equation}
This holds on every outcome and simultaneously for all $y\in\Kset$, so it also holds when
$y$ is replaced by a random point of $\Kset$, in particular by $U$.

We apply \eqref{eq:exp-ogd} to two sequences. The iterates of the algorithm satisfy
$x_{t+1}=\Pi_\Kset(x_t-\eta_td_t)$, so \eqref{eq:exp-ogd} with $y_t=x_t$, $g_t=d_t$, and
$y=U$ gives
\begin{equation}
\label{eq:primal-pf-x}
\sum_{t=1}^n\ip{d_t}{x_t-U}\le\frac{D^2}{2\eta_n}+\frac12\sum_{t=1}^n\eta_t\|d_t\|^2.
\end{equation}
For the analysis only, define $z_1=x_1$ and $z_{t+1}=\Pi_\Kset(z_t+\eta_t\xi_t)$. These
points lie in $\Kset$ but are not computed by the learner, since $\xi_t$ involves the
unobserved $p_t,q_t$. Applying \eqref{eq:exp-ogd} with $y_t=z_t$, $g_t=-\xi_t$, and $y=U$,
\begin{equation}
\label{eq:primal-pf-z}
\sum_{t=1}^n\ip{\xi_t}{U-z_t}\le\frac{D^2}{2\eta_n}+\frac12\sum_{t=1}^n\eta_t\|\xi_t\|^2.
\end{equation}

We now bound the expected right sides. By the definition of $\eta_n$,
$D^2/(2\eta_n)=D^2\cdot2\Gh\sqrt{S_n}/(2D)=D\Gh\sqrt{S_n}=L\sqrt{S_n}$.
For vectors, $\|a+b\|^2\le2\|a\|^2+2\|b\|^2$. Using this with $a=\tilde p_t$ and
$b=w_t\tilde q_t$, pulling out the $\Fs_t$-measurable $w_t^2$, and applying
\eqref{eq:finite-variance},
\begin{equation}
\label{eq:primal-pf-moment}
\E[\|d_t\|^2\mid\Fs_t]
\le2\E[\|\tilde p_t\|^2\mid\Fs_t]+2w_t^2\E[\|\tilde q_t\|^2\mid\Fs_t]
\le2\Gh^2(1+w_t^2)=2\Gh^2(S_t-S_{t-1}).
\end{equation}
This uses no assumption on the dependence between $\tilde p_t$ and $\tilde q_t$. Since
$\mu_t$ is $\Fs_t$-measurable and, by \eqref{eq:feedback} and pull-out,
$\E[d_t\mid\Fs_t]=\mu_t$, expanding the square gives
\begin{equation}
\label{eq:primal-pf-xi}
\E[\xi_t\mid\Fs_t]=0,\qquad
\E[\|\xi_t\|^2\mid\Fs_t]=\E[\|d_t\|^2\mid\Fs_t]-2\ip{\mu_t}{\mu_t}+\|\mu_t\|^2
\le\E[\|d_t\|^2\mid\Fs_t].
\end{equation}
Since $\eta_t$ is $\Fs_t$-measurable and nonnegative, the tower and pull-out properties,
\eqref{eq:primal-pf-moment}, and \eqref{eq:primal-pf-xi} give, for $g_t\in\{d_t,\xi_t\}$,
\begin{equation}
\frac12\E[\eta_t\|g_t\|^2]=\frac12\E\big[\eta_t\E[\|g_t\|^2\mid\Fs_t]\big]
\le\E\Big[\frac{D}{2\Gh\sqrt{S_t}}\Gh^2(S_t-S_{t-1})\Big]
=\frac L2\E\Big[\frac{S_t-S_{t-1}}{\sqrt{S_t}}\Big].
\end{equation}
For $0<b'\le b$,
$\frac{b-b'}{\sqrt b}=\frac{(\sqrt b-\sqrt{b'})(\sqrt b+\sqrt{b'})}{\sqrt b}
\le2(\sqrt b-\sqrt{b'})$. Applying this with $b=S_t$, $b'=S_{t-1}$ and telescoping,
\begin{equation}
\label{eq:primal-pf-sqrt}
\sum_{t=1}^n\frac{S_t-S_{t-1}}{\sqrt{S_t}}\le2(\sqrt{S_n}-\sqrt{S_0})
=2(\sqrt{S_n}-\alpha).
\end{equation}
Combining the identity $D^2/(2\eta_n)=L\sqrt{S_n}$, the last two displays, and
$L\E(\sqrt{S_n}-\alpha)\le L\E\sqrt{S_n}$, the expected right side of each of
\eqref{eq:primal-pf-x} and \eqref{eq:primal-pf-z} is at most $2L\E\sqrt{S_n}$.

The cross term in the decomposition below has mean zero. The point $z_t$ is computed from $x_1$, the
predictable step sizes $\eta_1,\ldots,\eta_{t-1}$, and $\xi_1,\ldots,\xi_{t-1}$, and $\xi_s=d_s-\mu_s$ is $\mathcal G_s$-measurable because
$\mu_s$ is $\Fs_s$-measurable and $\Fs_s\subseteq\mathcal G_s$. Hence $z_t$ is
$\mathcal G_{t-1}$-measurable, and so is $x_t$; both are $\Fs_t$-measurable. By pull-out
and \eqref{eq:primal-pf-xi}, and since the product is integrable (Lemma~\ref{lem:exp-clip}),
\begin{equation}
\label{eq:primal-pf-cross}
\E\ip{\xi_t}{z_t-x_t}=\E\ip{\E[\xi_t\mid\Fs_t]}{z_t-x_t}=0.
\end{equation}

Finally, we relate the left sides to the regret. Since $d_t=\mu_t+\xi_t$, on every outcome
\begin{equation}
\ip{\mu_t}{x_t-U}=\ip{d_t}{x_t-U}+\ip{\xi_t}{U-z_t}+\ip{\xi_t}{z_t-x_t}.
\end{equation}
Convexity of $\ell_t$ with $p_t\in\partial\ell_t(x_t)$ gives
$\ell_t(U)\ge\ell_t(x_t)+\ip{p_t}{U-x_t}$, that is,
$\ip{p_t}{x_t-U}\ge\ell_t(x_t)-\ell_t(U)$, and Lemma~\ref{lem:anchor} with $w_t\ge0$ gives
$w_t\ip{q_t}{x_t-U}\ge w_tv_t^\circ$ almost surely. Summing over $t\le n$,
\begin{equation}
\label{eq:primal-pf-lower}
\Reg_n(U)+\sum_{t=1}^nw_tv_t^\circ\le\sum_{t=1}^n\ip{\mu_t}{x_t-U}
=\sum_{t=1}^n\ip{d_t}{x_t-U}+\sum_{t=1}^n\ip{\xi_t}{U-z_t}+\sum_{t=1}^n\ip{\xi_t}{z_t-x_t}.
\end{equation}

Finally, take expectations in \eqref{eq:primal-pf-lower}. The first two sums on the right
have expectations at most $2L\E\sqrt{S_n}$ each, by \eqref{eq:primal-pf-x},
\eqref{eq:primal-pf-z}, and the bound on the expected right sides, and the third has
expectation $0$ by \eqref{eq:primal-pf-cross}. Thus
$\E\Reg_n(U)+\E\sum_tw_tv_t^\circ\le4L\E\sqrt{S_n}$, and
$\sqrt{S_n}=\sqrt{n+H_n}\le\sqrt n+\sqrt{H_n}$ proves \eqref{eq:exp-primal}.
If $U$ is a deterministic element of $\Ucal$, then $\E\ip{\xi_t}{x_t-U}=0$ directly, the
auxiliary sequence is not needed, and the constant $4L$ improves to $2L$.
\end{proof}

\subsubsection{Drift of the reset potential}

\begin{lemma}[Drift with a zero boundary derivative]
\label{lem:exp-drift}
For every round,
\begin{equation}
\label{eq:exp-drift}
\E[\Phi_\alpha(Q_t)-\Phi_\alpha(Q_{t-1})\mid\Fs_t]
\le w_tv_t^\circ-\tfrac12w_tr_t+202\alpha^2B_v^2.
\end{equation}
\end{lemma}

\begin{proof}
Extend the potential to $\R$ by $F(y)=\Phi_\alpha([y]_+)=e^{\alpha\pos y}-\alpha\pos y-1$.
For $y>0$, $F'(y)=\alpha(e^{\alpha y}-1)$ and $F''(y)=\alpha^2e^{\alpha y}$; for $y<0$,
$F'(y)=F''(y)=0$. Both one-sided derivatives at $0$ equal
$\alpha(e^0-1)=0$, so $F$ is continuously differentiable with
\begin{equation}
\label{eq:drift-pf-derivs}
F'(y)=\alpha(e^{\alpha[y]_+}-1)\ge0,\qquad
0\le F''(y)\le\alpha^2e^{\alpha y}\quad(y\ne0).
\end{equation}
In particular $F$ is nondecreasing, and $F'$ is locally absolutely continuous (it is
$C^1$ on each side of $0$ and continuous at $0$).

Let $y\ge0$ and $h\in\R$ with $\alpha h\le1$. Taylor's formula with integral remainder
(valid because $F'$ is absolutely continuous on the segment from $y$ to $y+h$) gives
\begin{equation}
F(y+h)-F(y)=F'(y)h+h^2\int_0^1(1-\theta)F''(y+\theta h)\,d\theta.
\end{equation}
For $\theta\in[0,1]$, $\alpha\theta h\le\max\{\alpha h,0\}\le1$, so by
\eqref{eq:drift-pf-derivs}, $F''(y+\theta h)\le\alpha^2e^{\alpha y}e^{\alpha\theta h}
\le e\,\alpha^2e^{\alpha y}$. Since $\int_0^1(1-\theta)\,d\theta=\frac12$ and
$F'(y)=\alpha(e^{\alpha y}-1)$,
\begin{equation}
\label{eq:exp-taylor}
F(y+h)-F(y)\le\alpha(e^{\alpha y}-1)h+\tfrac32\alpha^2e^{\alpha y}h^2,
\end{equation}
because $e/2<3/2$. The bound also holds for steps that cross zero, where the reset
enters only through the quadratic term.

Since $Q_t=[Q_{t-1}+\bar v_t-r_t]_+$ and $F(y)=\Phi_\alpha([y]_+)$, we have
$\Phi_\alpha(Q_t)=F(Q_{t-1}+\bar v_t-r_t)$; also $\Phi_\alpha(Q_{t-1})=F(Q_{t-1})$
because $Q_{t-1}\ge0$. As $F$ is nondecreasing and $\bar v_t\le\bar v_t^\circ$
(Lemma~\ref{lem:exp-clip}),
\begin{equation}
\label{eq:drift-pf-dominate}
\Phi_\alpha(Q_t)-\Phi_\alpha(Q_{t-1})
\le F(Q_{t-1}+\delta_t)-F(Q_{t-1}),\qquad \delta_t=\bar v_t^\circ-r_t.
\end{equation}

Since $r_t\ge0$ and $\bar v_t^\circ\le1/\alpha$, $\alpha \delta_t\le1$; so
\eqref{eq:exp-taylor} applies with $y=Q_{t-1}$, $h=\delta_t$. By the algorithm,
$\alpha(e^{\alpha Q_{t-1}}-1)=w_t$, and therefore
$\alpha^2e^{\alpha Q_{t-1}}=\alpha(w_t+\alpha)$. Using also
$\delta_t^2\le2(\bar v_t^\circ)^2+2r_t^2$,
\begin{equation}
\label{eq:drift-pf-path}
\Phi_\alpha(Q_t)-\Phi_\alpha(Q_{t-1})
\le w_t(\bar v_t^\circ-r_t)+3\alpha(w_t+\alpha)\big((\bar v_t^\circ)^2+r_t^2\big).
\end{equation}

The quantities $w_t$ and $r_t$ are $\Fs_t$-measurable and $w_t\ge0$. By
\eqref{eq:exp-clip-bias}, $\E[\bar v_t^\circ\mid\Fs_t]\le v_t^\circ+\alpha B_v^2$, and by
\eqref{eq:exp-clip-second}, $\E[(\bar v_t^\circ)^2\mid\Fs_t]\le B_v^2$. Taking
$\E[\,\cdot\mid\Fs_t]$ in \eqref{eq:drift-pf-path} and expanding
$3\alpha(w_t+\alpha)(B_v^2+r_t^2)$, the conditional drift is at most
\begin{equation}
\label{eq:exp-drift-raw}
w_tv_t^\circ-w_tr_t+4\alpha B_v^2w_t+3\alpha w_tr_t^2
+3\alpha^2(B_v^2+r_t^2).
\end{equation}

By \eqref{eq:exp-r-bound}, $r_t\ge16\alpha B_v^2$, so $4\alpha B_v^2w_t\le\frac14w_tr_t$.
Also, by \eqref{eq:exp-r-bound} and $\alpha B_v\le1/128$,
\begin{equation}
3\alpha r_t\le3\alpha\cdot\tfrac{65}8B_v\le\frac{195}{1024}<\frac14,
\qquad\text{so}\qquad 3\alpha w_tr_t^2\le\tfrac14w_tr_t.
\end{equation}
Hence $-w_tr_t+4\alpha B_v^2w_t+3\alpha w_tr_t^2\le-\frac12w_tr_t$. Finally,
$r_t^2\le\frac{4225}{64}B_v^2$ gives
$3\alpha^2(B_v^2+r_t^2)\le3\cdot\frac{4289}{64}\alpha^2B_v^2=\frac{12867}{64}\alpha^2B_v^2
\le202\alpha^2B_v^2$. Substituting these bounds in \eqref{eq:exp-drift-raw} proves
\eqref{eq:exp-drift}. A potential with nonzero slope at $0$ would instead add a
first-order term at each reset.
\end{proof}

\subsection{\texorpdfstring{Proof of the master inequality~\eqref{eq:exp-master} in Theorem~\ref{thm:exp}}{Proof of the master inequality (17) in Theorem 4.1}}

\label{app:exp-master}

\begin{proof}[Proof of \eqref{eq:exp-master}]
Since $H_t-H_{t-1}=w_t^2=(\sqrt{H_t}-\sqrt{H_{t-1}})(\sqrt{H_t}+\sqrt{H_{t-1}})$ and
$H_{t-1}\ge H_0=\alpha^2>0$, the definition of $r_t$
gives
\begin{equation}
\label{eq:exp-cancel}
\tfrac12w_tr_t
=8\alpha B_v^2w_t+4L\frac{w_t^2}{\sqrt{H_t}+\sqrt{H_{t-1}}}
=8\alpha B_v^2w_t+4L(\sqrt{H_t}-\sqrt{H_{t-1}}).
\end{equation}
Summing over $t\le n$ and telescoping, with $\sqrt{H_0}=\alpha$,
\begin{equation}
\label{eq:master-pf-cancel-sum}
\tfrac12\sum_{t=1}^nw_tr_t=8\alpha B_v^2\sum_{t=1}^nw_t+4L\sqrt{H_n}-4L\alpha.
\end{equation}

Since $\Phi_\alpha(Q_0)=\Phi_\alpha(0)=0$, telescoping and the tower property give
$\E\Phi_\alpha(Q_n)=\sum_{t=1}^n\E\,\E[\Phi_\alpha(Q_t)-\Phi_\alpha(Q_{t-1})\mid\Fs_t]$.
By Lemma~\ref{lem:exp-drift},
\begin{equation}
\label{eq:master-pf-drift-sum}
\E\Phi_\alpha(Q_n)\le\E\sum_{t=1}^nw_tv_t^\circ-\tfrac12\E\sum_{t=1}^nw_tr_t
+202\alpha^2B_v^2n.
\end{equation}

Now add \eqref{eq:master-pf-drift-sum} to \eqref{eq:exp-primal}. The term
$\E\sum_tw_tv_t^\circ$ appears on the left of \eqref{eq:exp-primal} and on the right of
\eqref{eq:master-pf-drift-sum}, so it cancels:
\begin{equation}
\E\Reg_n(U)+\E\Phi_\alpha(Q_n)
\le4L\sqrt n+4L\E\sqrt{H_n}-\tfrac12\E\sum_{t=1}^n w_tr_t+202\alpha^2B_v^2n.
\end{equation}
Substituting the expectation of \eqref{eq:master-pf-cancel-sum}, the terms
$\pm4L\E\sqrt{H_n}$ cancel and
\begin{equation}
\E\Reg_n(U)+\E\Phi_\alpha(Q_n)
\le4L\sqrt n+4L\alpha-8\alpha B_v^2\E\sum_{t=1}^n w_t+202\alpha^2B_v^2n.
\end{equation}
Moving the weighted sum to the left proves \eqref{eq:exp-master}.
\end{proof}

\subsection{Auxiliary lemmas for the budget bound}
\label{app:exp-aux2}

The next two lemmas use the master inequality \eqref{eq:exp-master}, proved above.

\begin{lemma}[Exponential moments and the running maximum]
\label{lem:exp-max}
With $K_0=4+10L$ and $\Lambda_T=\log((T+1)K_0T)$,
\begin{equation}
\label{eq:exp-max}
\E e^{\alpha Q_n}\le K_0T\quad(0\le n\le T),\qquad
\E\max_{0\le n\le T}Q_n\le\Lambda_T/\alpha.
\end{equation}
\end{lemma}

\begin{proof}
Fix a comparator $U\in\Ucal$, which exists by assumption. Each loss difference satisfies
$|\ell_t(x_t)-\ell_t(U)|\le G\|x_t-U\|\le GD\le L$, so $\Reg_n(U)\ge-Ln$. The weighted sum
in \eqref{eq:exp-master} is nonnegative. Dropping it and using $\E\Reg_n(U)\ge-Ln$,
\begin{equation}
\label{eq:max-pf-phi}
\E\Phi_\alpha(Q_n)\le Ln+4L\sqrt n+4L\alpha+202\alpha^2B_v^2n.
\end{equation}

For $s\ge0$, $e^s\ge2s$ (the minimum of $e^s-2s$ is $2-2\log2>0$). With $s=\alpha Q$
for $Q\ge0$, this gives $e^{\alpha Q}\le2e^{\alpha Q}-2\alpha Q=2(1+\Phi_\alpha(Q))$.
Therefore, by \eqref{eq:max-pf-phi},
\begin{equation}
\label{eq:max-pf-exp}
\E e^{\alpha Q_n}\le2+2Ln+8L\sqrt n+8L\alpha+404\alpha^2B_v^2n.
\end{equation}

Let $1\le n\le T$. Use $\sqrt n\le n\le T$, $8L\alpha\le8B_v\alpha\le1/16$, and
$404\alpha^2B_v^2\le404/16384<0.03$. Then the right side of \eqref{eq:max-pf-exp} is at
most $2+\frac1{16}+0.03T+10LT\le4T+10LT=K_0T$, since $T\ge1$. For $n=0$,
$\E e^{\alpha Q_0}=1\le K_0T$.

Write $M_T=\max_{0\le n\le T}Q_n$. Because $s\mapsto e^{\alpha s}$ is increasing and all
terms are positive, $e^{\alpha M_T}=\max_ne^{\alpha Q_n}\le\sum_{n=0}^Te^{\alpha Q_n}$.
Taking expectations and using $\E e^{\alpha Q_n}\le K_0T$,
$\E e^{\alpha M_T}\le(T+1)K_0T$. Since $\log$ is concave, Jensen's inequality yields
\begin{equation}
\alpha\E M_T=\E\log e^{\alpha M_T}\le\log\E e^{\alpha M_T}\le\log((T+1)K_0T)=\Lambda_T.
\end{equation}
\end{proof}

\begin{lemma}[Total regularization cost]
\label{lem:exp-reg-cost}
\begin{equation}
\label{eq:exp-reg-cost}
\E\sum_{t=1}^T r_t\le16\alpha B_v^2T+8L\sqrt{2T\Lambda_T}.
\end{equation}
\end{lemma}

\begin{proof}
In the next display, the first inequality uses $\sqrt{H_{t-1}}\ge0$; the second is
Cauchy--Schwarz, $\sum_t1\cdot c_t\le\sqrt{T\sum_tc_t^2}$ with $c_t=w_t/\sqrt{H_t}$,
followed by the substitution $w_t^2=H_t-H_{t-1}$:
\begin{equation}
\label{eq:reg-pf-cs}
\sum_{t=1}^T\frac{w_t}{\sqrt{H_t}+\sqrt{H_{t-1}}}
\le\sum_{t=1}^T\frac{w_t}{\sqrt{H_t}}
\le\sqrt{T\sum_{t=1}^T\frac{H_t-H_{t-1}}{H_t}}.
\end{equation}
For $0<a\le b$, $\frac{b-a}b\le\log\frac ba$; this is $\log x\le x-1$ with $x=a/b$.
Applying it with $a=H_{t-1}\ge\alpha^2>0$, $b=H_t$, and telescoping,
\begin{equation}
\label{eq:reg-pf-log}
\sum_{t=1}^T\frac{H_t-H_{t-1}}{H_t}\le\sum_{t=1}^T\log\frac{H_t}{H_{t-1}}
=\log\frac{H_T}{H_0}=\log\frac{H_T}{\alpha^2}.
\end{equation}

Set $a_t=w_t/\alpha=e^{\alpha Q_{t-1}}-1\ge0$. Since $H_T=\alpha^2+\sum_tw_t^2$ and all
$a_t\ge0$,
\begin{equation}
\label{eq:reg-pf-square}
\frac{H_T}{\alpha^2}=1+\sum_{t=1}^Ta_t^2\le\Big(1+\sum_{t=1}^Ta_t\Big)^2 .
\end{equation}

By \eqref{eq:reg-pf-square}, Jensen's inequality for the concave $\log$, and
Lemma~\ref{lem:exp-max} (which gives $\E a_t=\E e^{\alpha Q_{t-1}}-1\le K_0T-1$ for
$t\le T$),
\begin{equation}
\label{eq:reg-pf-elog}
\E\log\frac{H_T}{\alpha^2}
\le2\E\log\Big(1+\sum_{t=1}^Ta_t\Big)
\le2\log\Big(1+\sum_{t=1}^T\E a_t\Big)
\le2\log(K_0T^2)\le2\Lambda_T,
\end{equation}
where we used $1+T(K_0T-1)\le K_0T^2\le(T+1)K_0T$.

Finally, by \eqref{eq:reg-pf-cs}, \eqref{eq:reg-pf-log}, Jensen's inequality for the concave
square root, and \eqref{eq:reg-pf-elog},
\begin{equation}
\E\sum_{t=1}^T\frac{w_t}{\sqrt{H_t}+\sqrt{H_{t-1}}}
\le\E\sqrt{T\log\frac{H_T}{\alpha^2}}
\le\sqrt{T\,\E\log\frac{H_T}{\alpha^2}}\le\sqrt{2T\Lambda_T}.
\end{equation}
Multiplying by $8L$ and adding the constant part $\sum_t16\alpha B_v^2=16\alpha B_v^2T$
of $r_t$ proves \eqref{eq:exp-reg-cost}.
\end{proof}

\subsection{\texorpdfstring{Proof of the budget bound~\eqref{eq:exp-wcv} and the rates~\eqref{eq:exp-rates} in Theorem~\ref{thm:exp}}{Proof of the budget bound (18) and the rates (19) in Theorem 4.1}}
\label{app:exp-budget}

\begin{proof}[Proof of \eqref{eq:exp-wcv} and \eqref{eq:exp-rates}]
Write $M_T=\max_{0\le n\le T}Q_n$ and
$N_n=\sum_{t=1}^n\varepsilon_t$, with $N_0=0$.

Since $[y]_+\ge y$, the update gives $Q_t\ge Q_{t-1}+\bar v_t-r_t$, that is,
$\bar v_t\le Q_t-Q_{t-1}+r_t$. Summing over a window $[t_1,t_2]\subseteq[1,T]$, the balance
terms telescope; using $Q_{t_1-1}\ge0$, $Q_{t_2}\le M_T$, and $r_t\ge0$,
\begin{equation}
\label{eq:transfer-pf-clipped}
\sum_{t=t_1}^{t_2}\bar v_t\le Q_{t_2}-Q_{t_1-1}+\sum_{t=t_1}^{t_2} r_t
\le M_T+\sum_{t=1}^T r_t.
\end{equation}

Since $\tilde v_t=v_t+\varepsilon_t$, we can write
$v_t=\bar v_t+(\tilde v_t-\bar v_t)-\varepsilon_t$. Moreover
$\tilde v_t-\bar v_t\le[\tilde v_t-\bar v_t]_+$, and
$-\sum_{t=t_1}^{t_2}\varepsilon_t=-(N_{t_2}-N_{t_1-1})
\le2\max_{0\le n\le T}|N_n|$. Together with
\eqref{eq:transfer-pf-clipped}, these give the following bound on every window sum
$\sum_{t=t_1}^{t_2}v_t$; since it is nonnegative, it also covers the maximum with $0$ in the
definition of $\WCV_T$. Hence, pathwise,
\begin{equation}
\label{eq:exp-window-transfer}
\WCV_T\le M_T+\sum_{t=1}^T r_t+\sum_{t=1}^T[\tilde v_t-\bar v_t]_+
+2\max_{0\le n\le T}|N_n|.
\end{equation}
Negative truncation residuals only reduce the window sum, so they are not charged.

By the tower property and \eqref{eq:exp-positive-residual}, the expected positive
truncation residual is at most $\alpha B_v^2T$. For the noise, $\varepsilon_t$ is
$\mathcal G_t$-measurable and, since $\mathcal G_{t-1}\subseteq\Fs_t$,
$\E[\varepsilon_t\mid\mathcal G_{t-1}]=\E[\E[\varepsilon_t\mid\Fs_t]\mid\mathcal G_{t-1}]=0$.
So $(N_n)$ is a martingale with respect to $(\mathcal G_n)$, square
integrable because $\E\varepsilon_t^2\le\sigma^2$. Its increments are orthogonal: for
$s<t$, $\varepsilon_s$ is $\mathcal G_{t-1}$-measurable, so
$\E[\varepsilon_s\varepsilon_t]=\E[\varepsilon_s\E[\varepsilon_t\mid\mathcal G_{t-1}]]=0$.
Hence $\E(N_T)^2=\sum_t\E\varepsilon_t^2\le\sigma^2T$. Doob's $L^2$ maximal
inequality~\citep{doob1953stochastic},
$\E\max_{n\le T}|N_n|^2\le4\E(N_T)^2$, and Jensen's inequality
give
\begin{equation}
\label{eq:transfer-pf-doob}
\E\max_{0\le n\le T}|N_n|
\le\sqrt{\E\max_{0\le n\le T}|N_n|^2}
\le2\sqrt{\E(N_T)^2}\le2\sigma\sqrt T.
\end{equation}

To prove \eqref{eq:exp-wcv}, take expectations in \eqref{eq:exp-window-transfer} and insert
Lemma~\ref{lem:exp-max} ($\E M_T\le\Lambda_T/\alpha$), Lemma~\ref{lem:exp-reg-cost},
the truncation bound $\alpha B_v^2T$, and \eqref{eq:transfer-pf-doob}:
\begin{equation}
\E\WCV_T\le\frac{\Lambda_T}\alpha+\big(16\alpha B_v^2T+8L\sqrt{2T\Lambda_T}\big)
+\alpha B_v^2T+4\sigma\sqrt T,
\end{equation}
which is \eqref{eq:exp-wcv}.

For the prescribed $\alpha=\min\{(128B_v)^{-1},\sqrt{\Lambda_T/(17T)}/B_v\}$,
\begin{equation}
\label{eq:alpha-admissible}
\alpha B_v\le\frac1{128}
\end{equation}
by construction, so the bounds above apply. The function
$f(\alpha)=\Lambda_T/\alpha+17\alpha B_v^2T$ is convex on $(0,\infty)$, with
$f'(\alpha)=-\Lambda_T/\alpha^2+17B_v^2T$, so its unconstrained minimizer is
$\sqrt{\Lambda_T/(17T)}/B_v$, with value $2B_v\sqrt{17T\Lambda_T}$. This minimizer
satisfies the cap $\alpha\le(128B_v)^{-1}$ exactly when $\Lambda_T/(17T)\le128^{-2}$, that
is, when $T\ge\frac{16384}{17}\Lambda_T$. Otherwise $f$ is decreasing on
$(0,(128B_v)^{-1}]$, the prescribed $\alpha$ is $(128B_v)^{-1}$, and
$f(\alpha)=128B_v\Lambda_T+\frac{17}{128}B_vT$. In both cases $f(\alpha)=F_T$, and inserting
this in \eqref{eq:exp-wcv} gives the budget bound in \eqref{eq:exp-rates}. In the second
case, $T<\frac{16384}{17}\Lambda_T$ gives $(\frac{17}{128}T)^2\le17T\Lambda_T$, so in both
cases
\begin{equation}
\label{eq:alpha-terms}
\frac{\Lambda_T}\alpha+17\alpha B_v^2T=F_T\le2B_v\sqrt{17T\Lambda_T}+128B_v\Lambda_T.
\end{equation}

For the regret, drop the nonnegative potential and weighted-sum terms in
\eqref{eq:exp-master} with $n=T$, and use $4L\alpha\le4B_v\alpha\le\frac1{32}$ and
$202\alpha^2B_v^2T\le\frac{202}{17}\Lambda_T$ (from $\alpha^2B_v^2\le\Lambda_T/(17T)$):
\begin{equation}
\E\Reg_T(U)\le4L\sqrt T+4L\alpha+202\alpha^2B_v^2T
\le4L\sqrt T+\frac1{32}+\frac{202}{17}\Lambda_T.
\end{equation}

For the orders, $\Lambda_T=\log((T+1)K_0T)\le\log(2K_0T^2)=\log(2K_0)+2\log T$, and
$1\le\log(eT)$, $\log T\le\log(eT)$, so
\begin{equation}
\label{eq:transfer-pf-Lambda}
\Lambda_T\le c_\Lambda\log(eT),\qquad c_\Lambda=2+\log(2K_0).
\end{equation}
The function $h\mapsto\log(eh)/\sqrt h$ is maximized at $h=e$ with value
$2/\sqrt e<2$, so
\begin{equation}
\label{eq:log-sqrt}
\log(eh)\le2\sqrt h\qquad(h\ge1).
\end{equation}
Hence $\frac{202}{17}\Lambda_T\le\frac{404}{17}c_\Lambda\sqrt T$ and the regret is at most
$(4L+\frac1{32}+\frac{404}{17}c_\Lambda)\sqrt T=O(\sqrt T)$. For budget violation,
$\log(eT)\le T$ gives $\Lambda_T\le c_\Lambda\log(eT)\le c_\Lambda\sqrt{T\log(eT)}$, and
$\sqrt T\le\sqrt{T\log(eT)}$, so by \eqref{eq:alpha-terms} the budget bound in
\eqref{eq:exp-rates} is at most
\begin{equation}
\label{eq:exp-wcv-explicit}
\Big(2B_v\sqrt{17c_\Lambda}+8L\sqrt{2c_\Lambda}+4\sigma+128c_\Lambda B_v\Big)
\sqrt{T\log(eT)}=O\big(\sqrt{T\log(eT)}\big),
\end{equation}
where the constant depends only on $(\Gh,D,\sigma)$. This proves \eqref{eq:exp-rates}.
\end{proof}

\section{\texorpdfstring{Proof of Corollary~\ref{cor:anytime}}{Proof of Corollary 4.2}}
\label{app:anytime}

\begin{proof}[Proof of Corollary~\ref{cor:anytime}]
Use epochs $I_j=\{2^j,\ldots,2^{j+1}-1\}$ of planned length $h_j=2^j$,
$j=0,1,\ldots$. At the first round of each epoch set
$\alpha_j=\min\{(128B_v)^{-1},\sqrt{\Lambda_{h_j}/(17h_j)}/B_v\}$, the prescribed
parameter of Algorithm~\ref{alg:ledger-exp} for horizon $h_j$, reset $Q_0=0$ and
$H_0=\alpha_j^2$, and choose an action in $\Kset$.
Use local round numbers in all other updates. Fix a deterministic horizon $T$ and a
random action $U\in\Kset$ that is almost surely feasible on rounds $1,\ldots,T$. Let $J=\lfloor\log_2T\rfloor$.
Since $2^J\le T<2^{J+1}$, the epochs $I_0,\ldots,I_{J-1}$ are complete before round $T$,
epoch $J$ is the last nonempty one, and it may be unfinished. Let
$n_j=|I_j\cap[1,T]|\le h_j$ be the number of rounds of epoch $j$ that are played.

Inside epoch $j$ the algorithm is Algorithm~\ref{alg:ledger-exp} with horizon $h_j$ and
parameter $\alpha_j$, started from an action chosen by the learner. The proofs of
Appendix~\ref{app:exp} use only: (a) the per-round conditions \eqref{eq:feedback} and
\eqref{eq:finite-variance} given $\Fs_t$; (b) the fact that $w_t,H_t,\eta_t,r_t$ are
computed before the round-$t$ feedback; (c) the geometry: all actions, the auxiliary points of Lemma~\ref{lem:exp-primal}, and
$U$ lie in $\Kset$, of diameter at most $D$, and projection onto $\Kset$ is nonexpansive;
(d) the inequality $\alpha B_v\le1/128$, which holds for $\alpha_j$ by
\eqref{eq:alpha-admissible} with $T=h_j$; and
(e) feasibility of $U$ on the rounds being summed, through Lemma~\ref{lem:anchor}.
Conditions (a)--(c) hold after any history and are unaffected by deterministic
restarts; the pathwise inequality \eqref{eq:exp-ogd} holds for all points of $\Kset$ at
once, so the same random $U$ can be used in every epoch. For (e), a prefix of $n\le h_j$
rounds of epoch $j$ only involves rounds up to $T$, where $U$ is feasible. Hence the bounds of Appendix~\ref{app:exp} hold for the first
$n$ rounds of each epoch, with $(h_j,\alpha_j)$ in place of $(T,\alpha)$. These are
unconditional expectations taken under the global filtration $(\Fs_t,\mathcal G_t)$:
the restart times are deterministic, so the same tower-property arguments apply with
the time index running over the rounds of the epoch, and the bounds for different epochs
can be added.

Consider one epoch, and write $h=h_j$ and $\alpha_h=\alpha_j$. For any prefix of
$n\le h$ rounds,
\eqref{eq:exp-master} (with the nonnegative terms dropped) gives
\begin{equation}
\label{eq:any-reg-epoch}
\E\Reg_n(U)\le4L\sqrt n+4L\alpha_h+202\alpha_h^2 B_v^2 n
\le4L\sqrt h+\frac1{32}+\frac{202}{17}\Lambda_h
\le C_R\sqrt h,
\end{equation}
where we used $n\le h$, $4L\alpha_h\le1/32$, $\alpha_h^2B_v^2h\le\Lambda_h/17$,
\eqref{eq:transfer-pf-Lambda} with $T=h$, and \eqref{eq:log-sqrt}. One may take
$C_R=4L+\frac1{32}+\frac{404}{17}c_\Lambda$, using $\sqrt h\ge1$.

For budget violation within the epoch, Lemmas~\ref{lem:exp-max} and~\ref{lem:exp-reg-cost} and the transfer
\eqref{eq:exp-window-transfer} use the horizon only as the range of summation and
through $\alpha B_v\le1/128$. Repeating them through $n\le h$ rounds, with the same
$\alpha_h$, gives
\begin{equation}
\label{eq:any-wcv-epoch}
\E\WCV_n\le\frac{\Lambda_n}{\alpha_h}+17\alpha_h B_v^2n
+8L\sqrt{2n\Lambda_n}+4\sigma\sqrt n,\qquad
\Lambda_n=\log((n+1)K_0n).
\end{equation}
Each term is nondecreasing in $n$, so we may replace $n$ by $h$. Then
\eqref{eq:alpha-terms} with $T=h$, $\Lambda_h\le c_\Lambda\log(eh)$, $\log(eh)\le h$, and
$\sqrt h\le\sqrt{h\log(eh)}$ give
\begin{equation}
\label{eq:any-wcv-epoch-const}
\E\WCV_n\le C_W\sqrt{h\log(eh)},\qquad
C_W=2B_v\sqrt{17c_\Lambda}+8L\sqrt{2c_\Lambda}+4\sigma+128c_\Lambda B_v .
\end{equation}

The sets $I_j\cap[1,T]$, $0\le j\le J$, partition the rounds $1,\ldots,T$, so
$\Reg_T(U)=\sum_{j=0}^J\Reg^{(j)}$, where $\Reg^{(j)}$ is the regret over the $n_j$ played
rounds of epoch $j$. By \eqref{eq:any-reg-epoch} with $h=2^j$,
\begin{equation}
\label{eq:any-reg-sum}
\E\Reg_T(U)\le C_R\sum_{j=0}^J\sqrt{2^j}.
\end{equation}

For budget violation, let $[t_1,t_2]\subseteq[1,T]$ be any window. Its intersection with each $I_j$ is either
empty or a contiguous window inside epoch $j$, and these pieces partition $[t_1,t_2]$. The
sum of $v_t$ over each nonempty piece is at most $\WCV^{(j)}$, the budget violation of the
played rounds of epoch $j$. Summing the pieces and maximizing over windows gives,
pathwise,
\begin{equation}
\label{eq:any-wcv-split}
\WCV_T\le\sum_{j=0}^J\WCV^{(j)}_{n_j},
\qquad\text{so}\qquad
\E\WCV_T\le C_W\sum_{j=0}^J\sqrt{2^j\log(e2^j)},
\end{equation}
by \eqref{eq:any-wcv-epoch-const}.

For the two geometric sums, $2^J\le T$ gives
\begin{equation}
\sum_{j=0}^J\sqrt{2^j}=\frac{2^{(J+1)/2}-1}{\sqrt2-1}
\le\frac{\sqrt2\,2^{J/2}}{\sqrt2-1}
=\frac{2^{J/2}}{1-2^{-1/2}}\le\frac{\sqrt T}{1-2^{-1/2}}.
\end{equation}
Moreover $\log(e2^j)\le\log(eT)$ for $j\le J$, so
\begin{equation}
\sum_{j=0}^J\sqrt{2^j\log(e2^j)}\le\sqrt{\log(eT)}\sum_{j=0}^J\sqrt{2^j}
\le\frac{\sqrt{T\log(eT)}}{1-2^{-1/2}}.
\end{equation}
Inserting these sums in \eqref{eq:any-reg-sum} and \eqref{eq:any-wcv-split} gives
$\E\Reg_T(U)\le C_R\sqrt T/(1-2^{-1/2})=O(\sqrt T)$ and
$\E\WCV_T\le C_W\sqrt{T\log(eT)}/(1-2^{-1/2})=O(\sqrt{T\log(eT)})$.
\end{proof}

\section{Impact on finite horizons}
\label{app:finite}

Theorem~\ref{thm:exp} states its bounds with explicit constants. This appendix evaluates
them at moderate horizons by substituting the parameters into closed-form expressions;
the numbers describe the proved bounds, not the realized regret or violation of \LEDGER.

\subsection{The bounds and their crossover horizons}
\label{app:finite-bounds}

With $L$, $B_v$, $K_0$, $\Lambda_T$, and $\alpha$ as in Algorithm~\ref{alg:ledger-exp},
Theorem~\ref{thm:exp} gives $\E\Reg_T(U)\le\bar R_T$ for every comparator
$U\in\Ucal$ and $\E\WCV_T\le\bar W_T$, where
\begin{align}
\bar R_T&=4L\sqrt T+\tfrac1{32}+202\min\Big\{\frac{T}{16384},\frac{\Lambda_T}{17}\Big\},
\label{eq:fin-reg}\\
\bar W_T&=\frac{\Lambda_T}{\alpha}+17\alpha B_v^2T
+8L\sqrt{2T\Lambda_T}+4\sigma\sqrt T.
\label{eq:fin-wcv}
\end{align}
The regret bound is the master inequality \eqref{eq:exp-master} with $n=T$, using
$4L\alpha\le\frac1{32}$ and $\alpha^2B_v^2T=\min\{T/16384,\Lambda_T/17\}$; it is at most the
regret bound in \eqref{eq:exp-rates}. In \eqref{eq:fin-wcv}, the four terms are the running
maximum of the balance (Lemma~\ref{lem:exp-max}), the truncation residual plus the constant
part of the regularizers, the adaptive part of the regularizers
(Lemma~\ref{lem:exp-reg-cost}), and the value-noise martingale
\eqref{eq:transfer-pf-doob}; the first two sum to $F_T$.

Since $v_t\le L$ (Lemma~\ref{lem:anchor}) and $|\ell_t(x_t)-\ell_t(U)|\le L$, every learner
has $\WCV_T\le LT$ and $\Reg_T(U)\le LT$. Let $T_\times^{\Reg}$ and $T_\times^{\WCV}$ be the
smallest integers $T\ge2$ with $\bar R_T\le LT$ and $\bar W_T\le LT$. The ratios
$\bar R_T/(LT)$ and $\bar W_T/(LT)$ are nonincreasing in $T$, so each bound is informative
exactly from its crossover horizon on: $\Lambda_T/T$ is decreasing because
$\Lambda_T\ge\log 8>2$, and $F_T/T=\min_{0<a\le(128B_v)^{-1}}\{\Lambda_T/(aT)+17aB_v^2\}$
(Appendix~\ref{app:exp-budget}) is a minimum of nonincreasing functions of $T$.

\paragraph{A reference instance.}
Take $D=\Gh=\sigma=1$, so $L=1$ and $B_v=\sqrt2$. The cap $\alpha\le(128B_v)^{-1}$ is active
for $T<21{,}799$. Table~\ref{tab:fin-reference} lists the four terms of \eqref{eq:fin-wcv}
and both bounds relative to $LT$.

\begin{table}[h]
\centering\small
\renewcommand{\arraystretch}{1.15}
\setlength{\tabcolsep}{4pt}
\begin{tabular}{@{}rrrrrrrrr@{}}
\toprule
$T$ & $10^3\alpha$ & $\Lambda_T/\alpha$ & $17\alpha B_v^2T$ & $8L\sqrt{2T\Lambda_T}$
& $4\sigma\sqrt T$ & $\bar W_T$ & $\bar W_T/(LT)$
& $\bar R_T/(LT)$\\
\midrule
$10^3$ & 5.52 & 2{,}979 & 188 & 1{,}451 & 126 & 4{,}744 & 4.74 & 0.139\\
$10^4$ & 5.52 & 3{,}812 & 1{,}878 & 5{,}192 & 400 & 11{,}282 & 1.13 & 0.052\\
$10^5$ & 2.75 & 9{,}341 & 9{,}341 & 18{,}125 & 1{,}265 & 38{,}072 & 0.381 & 0.016\\
$10^6$ & 0.94 & 32{,}081 & 32{,}081 & 62{,}246 & 4{,}000 & 130{,}408 & 0.130 & 0.0044\\
$10^7$ & 0.32 & 108{,}893 & 108{,}893 & 211{,}283 & 12{,}649 & 441{,}717 & 0.044 & 0.0013\\
\bottomrule
\end{tabular}
\caption{Terms of the budget bound \eqref{eq:fin-wcv} and both bounds relative to the
trivial value $LT$, for $D=\Gh=\sigma=1$. Values are rounded.}
\label{tab:fin-reference}
\end{table}

Here $T_\times^{\WCV}=12{,}607$ and $T_\times^{\Reg}=17$ (the simplified regret expression in
\eqref{eq:exp-rates} would give $219$). At $T=10^4$ the adaptive regularizer accounts for
$46\%$ of the budget bound, the running maximum for $34\%$, truncation and the constant
regularizer for $17\%$, and the value noise for $3.5\%$; the cap on $\alpha$, which comes
from Lemma~\ref{lem:exp-drift}, inflates the running maximum. Once the cap is inactive, the
first two terms are equal, and together they are about as large as the adaptive
regularizer.

\subsection{Dependence on the noise level}
\label{app:finite-sigma}

The regret bound \eqref{eq:fin-reg} does not depend on $\sigma$. If the algorithm is run
with $\bar\sigma\ge\sigma$, then \eqref{eq:finite-variance} holds with $\bar\sigma$, so
Theorem~\ref{thm:exp} holds with the same regret bound and with \eqref{eq:fin-wcv} evaluated
at $\bar\sigma$; Table~\ref{tab:fin-sigma} therefore also gives the cost of a conservative
noise bound. If $\bar\sigma<\sigma$, Lemma~\ref{lem:exp-clip} can fail and
Theorem~\ref{thm:exp} gives no guarantee.

\begin{table}[h]
\centering\small
\renewcommand{\arraystretch}{1.15}
\begin{tabular}{@{}lrrrrrr@{}}
\toprule
$\sigma$ (or $\bar\sigma$) & 0.1 & 0.5 & 1 & 2 & 4 & 8\\
\midrule
$T_\times^{\WCV}$ & 8{,}667 & 9{,}793 & 12{,}607 & 22{,}402 & 57{,}899 & 189{,}972\\
$\bar W_T/(LT)$ at $T=10^5$ & 0.315 & 0.335 & 0.381 & 0.502 & 0.777 & 1.35\\
$\bar W_T/(LT)$ at $T=10^6$ & 0.108 & 0.115 & 0.130 & 0.172 & 0.265 & 0.460\\
\bottomrule
\end{tabular}
\caption{Budget crossover horizon and budget bound relative to $LT$ for $D=\Gh=1$ and
several noise levels. The regret crossover is $T_\times^{\Reg}=17$ in every column.}
\label{tab:fin-sigma}
\end{table}

\paragraph{Comparison with the joint lower bound.}
Corollary~\ref{cor:budget-joint} gives instances $A,B$ in $\mathfrak C_\sigma^{\mathrm G}$
($D=G=\Gh=1$, Gaussian value noise) on which every learner has
\[
\max\{\E_A\Reg_T(1),\E_B\WCV_T\}\ge\mathcal L_T=\frac{\min\{\sigma,1\}}4\sqrt{T\log T}.
\]
Since the action $1$ is a comparator in $A$, for \LEDGER\ the left side is at most
$\max\{\bar R_T,\bar W_T\}=\bar W_T$, and as $T\to\infty$, with $\Lambda_T/\log T\to2$,
\[
\frac{\bar W_T}{\mathcal L_T}\to\frac{8\sqrt{34}\,B_v+64L}{\min\{\sigma,1\}}
\qquad(\sigma\text{ fixed}),
\]
about $130$, $1{,}109$, and $11{,}065$ for $\sigma=1,0.1,0.01$. The ratio converges to a
constant for each fixed $\sigma>0$, uniformly over compact subsets of $(0,\infty)$, which is
the sense in which the horizon dependence matches (Section~\ref{sec:metric-consequences}).
As $\sigma\to0$ this lower bound vanishes; Remark~\ref{rem:vanishing-noise} combines it with
the unconstrained regret lower bound.

\paragraph{Summary.}
The inequalities of Theorem~\ref{thm:exp} hold at every horizon, but their constants are
large and were not optimized; the optimality statement concerns only the order of growth
in $T$, and the bounds are not safety certificates at moderate horizons.

\end{document}